\documentclass{article}
\PassOptionsToPackage{numbers,sort&compress}{natbib}
\usepackage[preprint]{neurips_2025}
\usepackage[utf8]{inputenc}
\DeclareUnicodeCharacter{FF08}{(}
\DeclareUnicodeCharacter{FF09}{)}
\usepackage[T1]{fontenc}
\usepackage{graphicx}
\usepackage{tikz}
\usepackage{pgfplots}
\pgfplotsset{compat=1.18}
\usepackage{float}
\usepackage{placeins}
\usepackage{xcolor}
\usepackage{fontawesome}
\usepackage{amsmath,amssymb,amsfonts}
\usepackage{amsthm}
\usepackage{booktabs}
\usepackage{fancyvrb}
\usepackage{fvextra}
\usetikzlibrary{shapes, arrows, positioning, shapes.callouts, arrows.meta}
\usetikzlibrary{calc}
\usepackage{capt-of}
\usepackage{needspace}
\usepackage{hyperref}

\usepackage[most]{tcolorbox}
\newtcolorbox{paperbox}[2][]{enhanced,breakable,
  colback=white,colframe=black!25,
  boxrule=0.6pt,arc=2pt,left=6pt,right=6pt,top=6pt,bottom=6pt,
  fonttitle=\bfseries,
  title={#2},#1}

\newcommand{\titlerunning}[1]{}
\newcommand{\authorrunning}[1]{}
\newcommand{\institute}[1]{}



\theoremstyle{plain}\newtheorem{theorem}{Theorem}
\newtheorem{proposition}{Proposition}
\newtheorem{lemma}{Lemma}
\newtheorem{corollary}{Corollary}
\theoremstyle{definition}\newtheorem{definition}{Definition}

\hypersetup{  pdftitle={Similarity Field Theory: A Mathematical Framework for Intelligence},
  pdfauthor={Kei-Sing Ng}}
\hypersetup{pdfsubject={Similarity fields and the breadth of generated exemplars}}
\newcommand{\weights}{\boldsymbol{p}}
\newcommand{\scoremat}{\mathbf{Z}}

\begin{document}

\title{Similarity Field Theory: A Mathematical Framework for Intelligence}
\titlerunning{Similarity Field Theory}

\author{Kei-Sing Ng\\\texttt{max.ksng.contact@gmail.com}}
\authorrunning{K.-S. Ng}

\institute{}

\maketitle

\begin{abstract}
How can a system be studied when its governing rules are unknown? We posit that persisting and transforming similarities underlie any comprehensible dynamic system. Similarity Field Theory starts with real, directed fields $S:U\times U\to[0,1]$ on entities $U$, requiring only $S(E,E)=1$. An observer's scores come from a stated protocol (e.g., human judgments or calibrated model responses). Further structural assumptions become measurable conditions rather than fixed premises: product transitivity, unit-vector geometry, and rankings may fit the same field to different degrees, supporting conditional conclusions. Complex unit-disk fields with unit diagonal extend this class. These domains need not be exhaustive. Fields may differ across observers and worlds; our bounds concern structure and transfer, not dynamics. Evolution is a sequence of fields and entity sets; concepts are entities inducing superlevel-set fibers. Intelligence is the ability to generate new exemplars of a concept from given ones. For $n$ distinct outputs other than $K$, two-sided fidelity at least $\alpha>0$ to concept $K$ and local product-transitivity degree $T_K>0$ give the sharp breadth bound $B\le n/[1+(n-1)T_K\alpha^2]$. Unit-vector overlap magnitudes in a common space give $B\le1/\alpha^2$ for one-sided fidelity at least $\alpha$; violations rule out this representation in any dimension. Agreement $r>0$ on a common finite set transfers membership at threshold $\alpha$ to $r\alpha$ across observers. Agreement on a tested case extends to an untested case with factor $rT_1T_2s^2$ under product transitivity, and exactly to the span of the tested cases under a unit-vector representation; both limits are sharp. The theory suggests a possible limit of AI alignment: better conformity to human-observable safety criteria need not improve alignment with the underlying safety concept.

\end{abstract}

\section{Introduction}
Given a set of entities, we take their graded, directed similarity as the primitive relation in our description of the world as comprehended by an observer.
Intelligence is the ability, given entities exemplifying a concept, to generate entities exemplifying the same concept.
We may regard the world as an evolving system. Within this system, similarity relations exist; without them, everything would collapse into undifferentiated chaos and the world we comprehend would cease to exist. Consider a cup before you: one second later it remains recognizably as a cup because certain similarities persist. If the cup falls and shatters, some similarities diminish while some others remain. These persisting and changing similarities jointly constitute the world we comprehend.

This paper formalizes this intuition by proposing a mathematical framework, \textbf{Similarity Field Theory}. The \hyperlink{field-terminology}{terminology note in Section~\ref*{sec:constructs}} explains this use of the word \emph{field}. Just as classical mechanics relies on calculus to analyze the continuous variation of physical quantities, Similarity Field Theory provides the essential constructs to analyze how intelligent systems preserve, modify, and generate similarity structures. We begin by establishing our foundational constructs—entities, concepts, and a similarity field. We then describe system evolution via a sequence $Z_p=(X_p,S^{(p)})$ and introduce a generative operator, culminating in a formal, operational definition of intelligence as the capacity to preserve conceptual structure under generation.

The framework is meant for systems whose rules are not known. The choice of further mathematical structure is itself an object of investigation, rather than a prerequisite for it. One records how the entities of such a system resemble one another, measures how nearly these scores satisfy a familiar structure, such as product transitivity, a unit-vector representation, a ranking or separated labels, and takes from that structure the conclusions that its degree of match supports. Several structures can be used on one system at once: each is matched to its own degree and contributes its own conclusions, and none has to hold exactly. Structures that are usually treated as alternatives, such as distances, vector representations and rankings, can therefore be combined in describing a single system (Section~\ref{sec:pred}). The definition of intelligence applies at every level, since a generative operator is itself an entity (Section~\ref{sec:learning}). On this view a comprehended world is built up step by step: each observer has its own field, generation adds entities to it, and agreement between fields measures what the observers share (Section~\ref{sec:observers}).

The table below places the fields studied here among real and complex scores and among general and unit-vector cases.

\par\noindent\begin{minipage}{\linewidth}\hypertarget{intro-field-table}{}\begin{center}
\small\begin{tabular}{@{}p{.14\linewidth}p{.39\linewidth}p{.38\linewidth}@{}}
\toprule
Scores & General field & Symmetric or Hermitian PSD field\\
\midrule
Real in $[0,1]$ &(1) Directed scores with diagonal one &(2) Real unit-vector overlaps in $[0,1]$\\
Complex in the unit disk &(4) Directed scores with magnitude and phase &(3) Complex unit-vector overlaps\\
\bottomrule
\end{tabular}
\end{center}\end{minipage}\par

The body of this paper works in the two real quadrants. A real similarity field is a complex similarity field whose values are non-negative reals. The entrywise modulus of a complex field is a real similarity field, so each result in the body applies to it whenever that result's own hypotheses hold for the modulus; the complex quadrants and what phase adds are treated in Section~\ref{sec:outlook}.

A core implication of the present framework is the working view that the world available to cognition is not given in wholly direct access, but is organized through an observer's perspective, similarity relations, and stable relational regularities. On this view, what models learn from human data is not the concept in itself, but an operationalized form of the concept as mediated by human observation, judgment, and description. This implication is especially consequential for safety: AI systems may be aligned less to safety as such than to human understandings and descriptions of safety. What humans identify, describe, or institutionalize as ``safe'' need not fully determine the underlying safety concept.

\paragraph{Contributions.} 
\begin{itemize}
\item A first-principles framework that treats graded, possibly asymmetric similarity as the primitive substrate for cognition and intelligence, with fibers as superlevel sets that define conceptual membership.

\item A sequence-based evolution formalism, including a stability theorem that characterizes anchors versus eventual confinement and an incompatibility theorem showing how asymmetry forbids reciprocal inclusion under cross-referenced thresholds.

\item An operational, generative definition of intelligence in terms of fiber preservation and coverage/fidelity measures.

\item A framework for interpretability that views networks as compositions of calibrated similarity fields.

\item A fidelity--breadth theorem (Theorem~\ref{thm:breadth}): under two-sided fidelity $\alpha$ and transitivity degree $T_K$ through the concept, the breadth of a generated batch is at most $1/(T_K\alpha^2)$, with a sharp finite-batch form and a dimension-free ceiling for unit-vector representations (Corollaries~\ref{cor:gram-breadth} and~\ref{cor:packing}).

\item A decomposition of directed scores into ranking and circulation (Theorems~\ref{thm:cycle} and~\ref{thm:hodge}), and bounds on what carries over between observers in terms of their agreement (Theorem~\ref{thm:transfer}).
\end{itemize}

\par\begingroup
\paragraph{What the framework claims.}
For generated entities that meet a positive threshold in both directions with a concept $K$, Theorem~\ref{thm:breadth} bounds how many effectively different outputs the batch can contain.

A writer might produce ten faithful variations on one example. Coverage counts the fraction that qualify, fidelity averages their scores, and breadth measures their mutual differences. Creativity in Section~\ref{sec:creativity-sft} concerns a further act, placing an entity in a new conceptual context. At two-sided threshold $\alpha>0$ and local transitivity degree $T_K>0$, breadth is at most $1/(T_K\alpha^2)$. Fidelity is required in both directions, whereas the fiber of Definition~\ref{def:fiber} constrains only $S(E,K)$. The ceiling is a limit, not a promise: a low $T_K$ permits breadth and does not produce it. A separate restriction holds for real unit-vector inner products. If each object scores at least $0.9$ towards one fixed concept and each pair of objects scores at most $0.7$, at most two such objects fit the representation.

An observer here is any system with a specified response to a set of entities. Richness (Definition~\ref{def:richness}) measures the distinctions that response makes on the set, as when a thermometer responds to temperature. Comparing a model at two training steps can reveal a change in its field, while its generative rule can also improve with the field fixed. On a finite set of at least three entities with strictly positive scores, directed differences in log scores admit a unique split into ranking and circulation. Some directionality may belong to the world and some may arise in observation. The split describes the recorded scores.

Agreement between fields on a common set is itself a symmetric, product-transitive similarity field. It bounds the change in a concept judgment. An observer's model of another field also gives a bound relating perceived and actual agreement. This second level may help describe a shared world when several observers are present. The framework leaves their number open. Because the field can differ from one observer or one world to another, the framework states no law for any single world. Its results are of two kinds: constraints that hold in every similarity field, and statements about what carries over when the field changes.

The framework does not supply the dynamics of any particular system. Those are to be measured, as a sequence of fields, and reproduced by a generative operator. What the framework supplies is the language in which the measurement is recorded, the criterion by which the reproduction is judged, and the limits that hold whatever the dynamics are. It does not give a procedure for finding such an operator, and Proposition~\ref{prop:complexcycle} shows that pairwise magnitudes alone need not determine everything about a system. The underlying claim throughout concerns distinctions available to observation. The world in itself, the physical origin of the given entities and the explanation of quantum mechanics lie outside that claim.
\par\endgroup

\section{Related Work}
Similarity Field Theory builds upon and extends several foundational areas. We position it as a first-principles framework that elevates similarity to a primitive substrate, complementing and unifying existing concepts rather than deriving from any single field.

\subsection{Fuzzy Set Theory: A Formal Analogy and an Essential Distinction}
A core structure in our framework—a grouping of entities based on their similarity to a given concept—is formally analogous to the `$\alpha$-cut' in Zadeh's Fuzzy Set Theory \cite{Zadeh1965}. This formal link provides a useful bridge to a mature field. However, this similarity is formal rather than foundational. The primary goal of Fuzzy Set Theory is to represent static knowledge under conditions of vagueness. In contrast, the central aim of Similarity Field Theory is to formulate a theory of intelligence based on the \emph{dynamics} of how conceptual structures evolve and are generated, which can reuse fuzzy constructs when convenient, without being reducible to them. Our fundamental primitive is the similarity relation itself, modeled as a directed field $S$. Our framework is therefore not an extension of Fuzzy Set Theory but a relation-centric dynamic similarity theory.

\subsection{Metric Spaces: Relaxed Axioms for Cognitive Modeling}\label{sec:metric-spaces}
Our framework deliberately departs from the rigid axioms of metric spaces. The requirement for a distance function to be symmetric and to obey the triangle inequality, while essential for describing physical space, is often too restrictive for modeling the nuances of cognitive or conceptual space. In practice, machine learning routinely employs measures such as cosine similarity, KL divergence, Jensen--Shannon divergence, and general Bregman divergences; several of these are asymmetric and/or violate the triangle inequality, motivating a relaxation of metric axioms for conceptual geometry \cite{Salton1975,KL1951,Lin1991,EndresSchindelin2003,Bregman1967,Banerjee2005,Tversky1977}. For the following examples, we fix a single similarity notion $S$ over the shared universe $U$ (e.g., cosine in a common embedding space or a corpus-based associative measure); all pairwise comparisons refer to this same $S$. Consider \emph{lion} (A), \emph{domestic cat} (B), and \emph{household pet} (C): $S(A,B)$ is high (both are felids), $S(B,C)$ is high (cats are prototypical pets), yet $S(A,C)$ is much lower (lions are not household pets), yielding high $S(A,B)$ and $S(B,C)$ but comparatively low $S(A,C)$.

\subsection{Category Theory: A Language of Structure}
Category theory provides a language for abstract structures, describing objects and their connections (morphisms). In the standard view, the link between two objects is a discrete, binary matter: a morphism either exists or does not. While powerful, this is insufficient to capture the graded nature of conceptual relations. We use a similarity field $S$ taking values in $[0,1]$ to quantify the strength of belonging/connection; given a concept $K$, we express structure and invariants via threshold sets of $S(\cdot,K)$; generation is ``intelligent'' when it adds new entities to these threshold sets within the system along the evolution. We do not read $S$ as a discrete morphism between arbitrary objects; instead, we center on the concept $K$ and study $S(\cdot,K)$ and its fibers $F_{\alpha}(K)$, using strength rather than mere existence to characterize belonging and connection. More importantly, this framework targets any system organized by similarity---from perception and learning to machine models, social bargaining, and institutional design---rather than being confined to categorical settings or specialist uses of ``fiber''. Enriched categorical phrasing is only an explanatory lens; our theory and results hold even without adopting it \cite{Lawvere1973,Kelly1982}.

\subsection{A Theoretical Foundation for Model Interpretability}

\subsubsection{Current Explorations and the Need for a Theory}
The academic community has made explorations into model interpretability, developing techniques to trace information ``circuits'' \cite{Lindsey2025}, identify ``persona vectors'' in activation space \cite{Chen2025}, and perform ``representation engineering'' \cite{Zou2023,Turner2023,Rimsky2024}. These empirical works enhance our ability to inspect AI systems, but they often lack a unifying mathematical theory that explains how these observed structures emerge from first principles. The Similarity Field Theory framework is proposed not to replace these techniques, but to offer a foundational, geometric mathematical perspective: it casts models as calibrated similarity fields in which each concept entity induces a fiber of examples, thereby unifying and guiding interpretability studies at a deeper structural level beyond purely statistical views.

\subsubsection{Beyond a Purely Statistical View}
Viewing a neural network as a statistical learning machine is a dominant and effective practice; it excels at characterizing in-distribution input--output behavior but offers limited insight into the formation of internal conceptual structure and into out-of-distribution (OOD) phenomena. Similarity Field Theory complements this view from the angle of conceptual similarity: we regard a trained model as a complex system composed of nested similarity fields and restate interpretability as a task of conceptual deconstruction---identifying latent concepts formed inside the model and defining them via their fibers (the sets of inputs that elicit high similarity responses). Consequently, the research goal shifts from merely fitting external statistics to understanding and organizing these concept fibers, enabling structured judgments beyond the training distribution in accordance with conceptual similarity rather than being constrained by frequency statistics of the training distribution. For example, if the training data include only ``red circles'' and ``green triangles'' and never ``green circles,'' the model can still classify a green circle as a circle because it matches along the fiber of the concept ``circle,'' rather than relying on frequency co-occurrence.

\begin{paperbox}{Box 1 \,|\, Why fiber-based structure can generalize when joint-frequency MLE collapses}
Consider two attributes---color $C\in\{\mathrm{red},\mathrm{green}\}$ and shape $H\in\{\mathrm{circle},\mathrm{triangle}\}$---and entities of the form $(C,H)$. Suppose the training set contains only $(\mathrm{red},\mathrm{circle})$ and $(\mathrm{green},\mathrm{triangle})$. A joint-frequency maximum-likelihood estimate assigns
\[
\widehat P_{\mathrm{MLE}}(C=\mathrm{green},H=\mathrm{circle})=\frac{n(\mathrm{green},\mathrm{circle})}{N}=0,
\]
so any decision rule that relies on the joint MLE as ``support'' cannot endorse the unseen combination $(\mathrm{green},\mathrm{circle})$ without introducing extra structure.

In Similarity Field Theory (SFT), a concept is an entity. Its fiber is the superlevel set of the directed score towards it. Let $K$ denote the concept ``circle'', and define the fiber $F_\alpha(K)=\{E:\ S(E,K)\ge \alpha\}$. Let $X_K$ be a small, nonempty exemplar set intended to lie in the fiber; in the simplest case, $X_K=\{x\}$ with $x=(\mathrm{red},\mathrm{circle})$. Define an exemplar-to-concept readout
\[
\hat S(E,K)=\max_{x\in X_K}\min\big\{P_K(E,x),\,S(x,K)\big\},
\]
where $P_K(\cdot,\cdot)$ denotes a $K$-aligned similarity probe (e.g., a shape-focused similarity). The estimated fiber is $\hat F_\alpha(K)=\{E:\hat S(E,K)\ge\alpha\}$.

\textbf{Lemma (one-step fiber extension).} If there exists $x\in X_K$ such that $S(x,K)\ge \alpha$ and $P_K(E,x)\ge \alpha$, then $\hat S(E,K)\ge \alpha$ and hence $E\in \hat F_\alpha(K)$.

Thus, even when $n(\mathrm{green},\mathrm{circle})=0$, the unseen entity $E=(\mathrm{green},\mathrm{circle})$ belongs to the estimated fiber $\hat F_\alpha(K)$ whenever it is sufficiently similar (in the $K$-aligned sense) to an exemplar already inside the fiber. Membership in $F_\alpha(K)$ would require an additional connection between the scores, for example $S(E,K)\ge\hat S(E,K)$. By contrast, additive smoothing (e.g., Laplace/Dirichlet) redistributes mass to unseen cells without preferentially selecting $(\mathrm{green},\mathrm{circle})$ over other unobserved combinations unless additional factorization or concept-specific structure is imposed.
\end{paperbox}

\par\begingroup
\subsection[Similarity as the basis of a comprehended world]{Similarity as the basis of a comprehended world}
Whose distinctions are available in a description of the world? Kant distinguishes objects as experienced from things as they are in themselves \cite{Kant1787}. In von Uexk\"ull's account of the \emph{Umwelt}, different organisms respond to different features of a common setting \cite{Uexkull1934}. Carnap starts one construction from a single subject's recollection of similarity between experiences, a relation with a direction in time \cite[Secs.~63--65, 78]{Carnap1928}. G.~A.~Kelly treats a person's contrasts as part of the way that person construes events \cite{GAKelly1955}. These accounts motivate specifying whose response a field represents and which comparisons that response can make.

Goodman describes the difficulty of imperfect community \cite[3rd ed., pp.~118--119]{Goodman1951}: each pair of items can share a quality even when the whole set shares none. In his example of three two-coloured discs, every pair shares a colour but no colour occurs on all three discs \cite[pp.~442--443]{Goodman1972}. Here a concept $K$ is specified as an entity and belongs to its own fiber for every threshold in $[0,1]$ by reflexivity. Membership tests a directed score towards $K$, and pairwise scores alone need not identify a common quality. If the threshold $\alpha$ also holds in reverse, the degree $T_K$ gives the lower bound $T_K\alpha^2$ between members on the specified finite set. One can ask the same question of any system for which the scores are available and compare two observers on common entities (Section~\ref{sec:observers}). Zadeh's similarity relations impose symmetry and max--min transitivity \cite{Zadeh1971}. We measure the degree of product transitivity of an observed field.

\subsection[Diversity, statistical, quantum, kernel and representational models]{Diversity, statistical, quantum, kernel and representational models}
In language generation, Banayeeanzade et al.\ relate a validity--diversity trade-off to the probabilities assigned to valid and invalid continuations \cite{Banayeeanzade2026}. Theorem~\ref{thm:breadth} uses two directions of fidelity and scores between outputs. At positive $\alpha$ and $T_K$ it bounds breadth by $1/(T_K\alpha^2)$ without assuming a token model. The realized breadth can remain constant as $T_K$ varies (Section~\ref{sec:pred}). Our diversity formula is that of Leinster and Cobbold, whose similarity matrices can already be asymmetric \cite{LeinsterCobbold2012}. For symmetric similarity matrices, under their hypotheses, Leinster and Meckes prove that the maximum diversity is the same at every order \cite{LeinsterMeckes2016}. Identity similarity gives the Hill numbers, with $D_1=\exp(H)$ at order one for Shannon entropy $H$ in natural units \cite{Hill1973}. Distillation uses a teacher's soft targets \cite{Hinton2015}. A Wasserstein loss makes the cost of a prediction depend on distances between labels \cite{Frogner2015}. Semantic uncertainty groups generated answers by meaning before computing entropy \cite{Kuhn2023}. Here the grouping, weights and ordered scores are explicit. The starting point is an observer's field together with a test on what a system generates.

Classical real-valued kernel methods use symmetric PSD kernels, realized as inner products of feature vectors \cite{Aronszajn1950,ScholkopfSmola2002}. Complex representations serve different purposes in related work. Trouillon et al.\ use a relation-specific complex score for directed links \cite{Trouillon2016}. M\"unch et al.\ use complex coordinates for generic proximities, including indefinite data \cite{Munch2021}. Magnetic eigenmaps encode graph direction in a Hermitian operator through phases \cite{Fanuel2018}. Pothos, Busemeyer and Trueblood model order-sensitive similarity judgments via sequential quantum projections \cite{Pothos2013}. Section~\ref{sec:outlook} considers one static matrix of overlaps and shows exactly what pairwise magnitudes can conceal about a cycle.

Representational similarity analysis compares the structures of pairwise dissimilarities in brains and models \cite{Kriegeskorte2008}. The Platonic Representation Hypothesis proposes convergence between representations learned by different models \cite{Huh2024}. A related question using our comparison is whether worst-pair agreement $r$ approaches one on a common finite set. Different alignment measures need not answer that same question.

B\v{e}lohl\'avek studies how perturbing fuzzy relations changes the truth degrees of logical formulas \cite{Belohlavek2007}. His comparison of relations and transitivity theorem give the product-case global estimate in Theorem~\ref{thm:transfer}(c). We also apply the pointwise argument at a specified middle entity $K$ and ask what happens to an observer's assessment of particular exemplars and generated outputs. This attribution distinguishes established results from their role in the generative definition.
\par\endgroup

\section{Foundational Constructs}\label{sec:constructs}

\begin{definition}[Entity]\label{def:entity}
An \emph{entity}, denoted $E$, is any identifiable element of discourse. An entity can be atomic, a concept, or itself a similarity relation.
\end{definition}

\begin{definition}[Universe]
The \emph{universe}, denoted $U$, is the set of all entities.
\end{definition}

\begin{definition}[Concept]
A \emph{concept}, denoted $K$, is an entity that serves as a representation for a class of other entities. It induces a unary map $S_K(E):=S(E,K)$ and thereby defines fibers as its superlevel sets.
\end{definition}

\begin{definition}[Similarity Field]\label{def:field}
A \emph{similarity field} is a mapping $S: U \times U \to [0,1]$ satisfying reflexivity: $S(E,E) = 1$ for all $E \in U$. We treat $S$ as a \emph{directed relational field}; symmetry and transitivity are not assumed.
\end{definition}

\begin{definition}[Fiber]\label{def:fiber}
Given a concept $K \in U$ and a threshold $\alpha \in [0,1]$, the \emph{fiber} induced by $K$ is the superlevel set:
\[
F_{\alpha}(K)=\{E \in U \mid S(E, K)\ge \alpha\}.
\]
\end{definition}

Definitions~\ref{def:field} and~\ref{def:fiber} specify the starting mathematical representation. In an application, a particular field is specified by the entity set, the observer or scoring system, and the ordered scores obtained under a stated comparison protocol. The ordered interval $[0,1]$ makes graded membership explicit through superlevel sets, while reflexivity ensures that $K\in F_\alpha(K)$ for every $\alpha\in[0,1]$. Section~\ref{sec:outlook} extends the value domain and relates it to the real-valued theory.

\noindent\textbf{Structural assumptions (minimal).} We work with the weakest topological or graph structures needed: a domain of observables $D\subseteq[0,1]^n$ admits well-defined limits and closures (e.g., by restricting to compact subsets), so that sequence limits, $\omega$-limit sets, and level sets are well-defined. Smoothness is assumed only when explicitly required.

\hypertarget{field-terminology}{}\noindent\textbf{Terminology clarification.} Throughout, ``field'' means a value-assigning function in the geometric sense; it is neither an algebraic field nor a physical field theory.

\subsection{On the Nature of the Similarity Field}\label{sec:field-nature}
It is a crucial feature of our framework that the similarity field $S$ is defined broadly, requiring only reflexivity. We deliberately do not enforce two stronger conditions common in classical mathematics:
\begin{itemize}
\item \textbf{Symmetry:} We do not assume that $S(E_1,E_2)=S(E_2,E_1)$. This allows the framework to model asymmetric cognitive phenomena in reality. Classic work in cognitive psychology has shown that people often judge, for example, ``North Korea is similar to China'' as more natural, and endorse it more frequently, than the reversed statement ``China is similar to North Korea'', even though the underlying pair of countries is the same; such reference-point and salience effects systematically break symmetry in human similarity judgments \cite{Tversky1977}.
\item \textbf{Transitivity:} High similarity need not be transitive. In conceptual space, consider \emph{lion} (A), \emph{domestic cat} (B), and \emph{household pet} (C): $S(A,B)$ is high (both are felids), $S(B,C)$ is high (domestic cats are prototypical pets), yet $S(A,C)$ is much lower (lions are not household pets), yielding high $S(A,B)$ and $S(B,C)$ but comparatively low $S(A,C)$. This pattern can depart from metric triangle-inequality assumptions (for $d=1-S$ it does so when $S(A,C)<S(A,B)+S(B,C)-1$) and matches empirical observations about graded category structure and conceptual similarity \cite{Tversky1977,Gardenfors2000}.
\end{itemize}
This deliberate relaxation is motivated by the arguments in Sections~\ref{sec:metric-spaces} and~\ref{sec:field-nature}, distinguishing Similarity Field Theory from frameworks based on Metric Spaces or strict Equivalence Relations. While a field satisfying all three properties is a valid special case, our broader definition is essential for modeling the continuous and overlapping nature of conceptual structures in the real world.

\subsection{Properties of Similarity Fields}
The properties of the similarity field, particularly its allowance for asymmetry, give rise to non-trivial principles governing conceptual membership. We formalize some of these properties below.

\begin{proposition}[Fiber monotonicity]\label{prop:fiber-monotone}
Fix $K\in U$. If $\alpha\ge\beta$, then $F_\alpha(K)\subseteq F_\beta(K)$.
\end{proposition}
\begin{proof}
If $E\in F_\alpha(K)$, then $S(E,K)\ge\alpha\ge\beta$, hence $E\in F_\beta(K)$.
\end{proof}

\begin{proposition}[Intersection identity]\label{prop:fiber-intersection}
For fixed $K\in U$ and thresholds $\{\alpha_r\}_{r=1}^m$ with $m\ge1$, we have $\bigcap_{r=1}^m F_{\alpha_r}(K)=F_{\max_r\alpha_r}(K)$.
\end{proposition}
\begin{proof}
If $E\in\bigcap_{r=1}^m F_{\alpha_r}(K)$, then $S(E,K)\ge\alpha_r$ for all $r$, so $S(E,K)\ge\max_r\alpha_r$ and $E\in F_{\max_r\alpha_r}(K)$. The reverse inclusion is immediate.
\end{proof}

\begin{proposition}
If $S$ is a similarity field on a nonempty universe $U$, $d = 1 - S$ is not.
\end{proposition}
\begin{proof}
For $d$ to be a similarity field, it must satisfy $d(E,E) = 1$. By definition, $d(E,E) = 1 - S(E,E) = 1 - 1 = 0$. This implies the contradiction $0 = 1$.
\end{proof}

\begin{proposition}[Closure under Multiplication]\label{prop:pointwise-product}
If $S_1, S_2$ are similarity fields, so is their pointwise product $S_1 \cdot S_2$.
\end{proposition}
\begin{proof}
The range remains $[0,1]$. For reflexivity, $(S_1 \cdot S_2)(E,E) = S_1(E,E) \cdot S_2(E,E) = 1 \cdot 1 = 1$.
\end{proof}

\begin{proposition}[Closure under Convex Combination]\label{prop:convex-mixture}
If $\{S_k\}$ is a collection of similarity fields and $\{w_k\}$ are non-negative weights summing to 1, then $\sum w_k S_k$ is a similarity field.
\end{proposition}
\begin{proof}
The range remains $[0,1]$. For reflexivity, $(\sum w_k S_k)(E,E) = \sum w_k S_k(E,E) = \sum w_k \cdot 1 = 1$.
\end{proof}

\begin{theorem}[Incompatibility Theorem]\label{thm:asymmetric}
Given two distinct entities $E_1, E_2 \in U$. Let their similarity values be $x = S(E_1, E_2)$ and $y = S(E_2, E_1)$, produced by the single similarity field $S$. If $x \neq y$ (i.e., the similarity relation is asymmetric), then the following two conditions cannot both be true simultaneously:
\begin{enumerate}
    \item $E_1$ belongs to the fiber induced by $E_2$ as a concept, with threshold $y$ (i.e., $E_1 \in F_y(E_2)$).
    \item $E_2$ belongs to the fiber induced by $E_1$ as a concept, with threshold $x$ (i.e., $E_2 \in F_x(E_1)$).
\end{enumerate}
\end{theorem}
\begin{proof}
We use proof by contradiction. Assume that both condition 1 and condition 2 are simultaneously true.
\begin{itemize}
    \item From condition 1 being true, by the definition of a fiber $F_y(E_2) = \{E \in U \mid S(E, E_2) \ge y\}$, we must have $S(E_1, E_2) \ge y$. Substituting $x = S(E_1, E_2)$, we get $x \ge y$.
    \item From condition 2 being true, by the definition of a fiber $F_x(E_1) = \{E \in U \mid S(E, E_1) \ge x\}$, we must have $S(E_2, E_1) \ge x$. Substituting $y = S(E_2, E_1)$, we get $y \ge x$.
\end{itemize}
Combining these two results, we have $x \ge y$ and $y \ge x$, which necessarily implies $x = y$. However, this contradicts the theorem's premise that $x \neq y$.
Therefore, our initial assumption is false, meaning that condition 1 and condition 2 cannot both be true simultaneously.
\end{proof}

\noindent\textbf{Significance of Theorem \ref{thm:asymmetric}:} The key point is that asymmetry induces a one-sided constraint: under cross-referenced thresholds, the two entities belong to each other's fibers exactly when their two directed scores are equal. This highlights a structural constraint: asymmetry can impose inherent limitations and a forced imbalance on conceptual membership. \textbf{It establishes the impossibility of reciprocal inclusion: if $E_1$'s perspective on $E_2$ differs from $E_2$'s perspective on $E_1$ (i.e., $x \neq y$), then they cannot simultaneously be members of each other's conceptual fibers when the threshold for inclusion is set by the other's standard.} In a stylized negotiation setting where each party sets its acceptance threshold using the other party’s valuation, the theorem yields a simple infeasibility condition: if the assessments are asymmetric, mutual agreement cannot be achieved under these cross-referenced thresholds unless at least one party adjusts its acceptance criterion, for example by slightly lowering a price (or raising an offer), which effectively shifts the fiber threshold and may, when the parties are near the boundary, turn non-membership into membership. In this reading the two directed scores of a single field are interpreted as the two parties' valuations of each other. Comparisons between the fields of two different observers are treated in Section~\ref{sec:observers}.

This principle is illustrated by the Novice ($N$) and Expert ($E$) scenario. Suppose $S(N, E) = 0.9$ and $S(E, N) = 0.2$. The Novice may belong to the Expert's fiber ($N \in F_{0.2}(E)$ since $0.9 \ge 0.2$), but the theorem guarantees the impossibility of the reverse: the Expert cannot belong to the Novice's fiber at the Novice's standard ($E \notin F_{0.9}(N)$ since $0.2 < 0.9$). Here only the Novice passes the other party's threshold. At a common threshold of $0.2$, both would qualify.

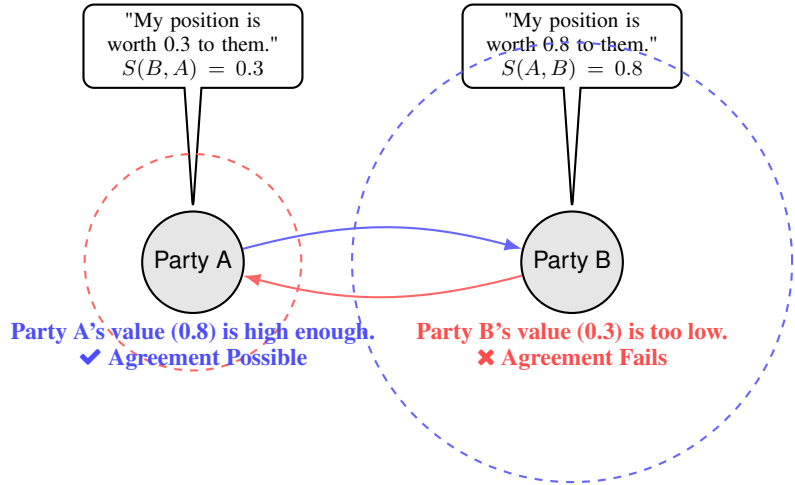
\begin{figure}[H]
\centering
\begin{tikzpicture}[x=1cm,y=1cm,font=\small\sffamily,
  party/.style={circle,draw=black!60,fill=gray!12,minimum size=7mm,inner sep=0pt},
  region/.style={dashed,very thick}]
  \begin{scope}
    \node[font=\small\sffamily\bfseries] at (0,1.40) {B's acceptance fiber};
    \draw[region,blue!60] (0,0) ellipse (1.82 and 0.89);
    \node[party] at (-0.78,0) {A};
    \node[party] at (0.78,0) {B};
    \node[text=blue!65!black] at (0,-1.23) {$S(A,B)=0.8\geq 0.3$};
    \node at (0,-1.72) {A qualifies for B.};
  \end{scope}
  \begin{scope}[xshift=5.05cm]
    \node[font=\small\sffamily\bfseries] at (0,1.40) {A's acceptance fiber};
    \draw[region,red!60] (-0.78,0) circle[radius=0.82cm];
    \node[party] at (-0.78,0) {A};
    \node[party] at (0.78,0) {B};
    \node[text=red!65!black] at (0,-1.23) {$S(B,A)=0.3<0.8$};
    \node at (0,-1.72) {B falls outside A's fiber.};
  \end{scope}
\end{tikzpicture}
\caption{The Incompatibility Theorem visualized through a negotiation. Each party's "acceptance fiber" is defined by the other's perceived value. Party B's offer ($S(B,A)=0.3$) is too low to enter Party A's acceptance fiber (threshold 0.8). Conversely, Party A's offer ($S(A,B)=0.8$) easily clears Party B's threshold (0.3). Because the condition is not met for both, mutual agreement is impossible. The left panel shows Party A inside Party B's fiber $F_{0.3}(B)$; the right panel shows Party B outside Party A's fiber $F_{0.8}(A)$.}
\label{fig:incompatibility_compact}
\end{figure}

\FloatBarrier

\section{System Evolution via Sequences}

\noindent\textbf{Sequence-only stance.} We do not introduce flow maps; all notions of evolution, stability, and $\omega$-limits are expressed purely in terms of sequences.

\begin{definition}[System-State Sequence]\label{def:state-seq}
A \emph{system-state sequence} is a sequence $Z_p=(X_p,S^{(p)})$ with $p=0,1,2,\dots$, where $X_p\subseteq U$ is a finite set of entities and $S^{(p)}:U\times U\to[0,1]$ is a similarity field at step $p$.
\end{definition}

\noindent\textbf{Index interpretation.} The index $p$ may denote discrete time, training steps\slash epochs, rounds of interaction, or any other well-ordered design parameter; our results depend only on its ordering, not on its physical interpretation. Even if an underlying continuous time exists, empirical measurements are effectively discrete at finite resolution; accordingly, we adopt a sequence-based formalism while remaining compatible with continuous-time models where available.

The evolution of a system-state sequence can be characterized by two primary modes:
\begin{enumerate}
    \item \textbf{Field Variation:} For a fixed set of entities, their relational structure changes along the sequence via $S^{(p)}(E_i,E_j)$. When $p$ samples a continuous underlying parameter, differentiability of $S^{(p)}$ may be assumed only when needed.
    \item \textbf{Discrete Composition Change:} The set $X_p$ evolves via the addition of entities. An entity that changes or ceases to matter is not deleted: its scores change under the first mode, and it leaves the fibers whose thresholds it no longer meets. This mode is driven by a generative operator (defined later).
\end{enumerate}

\noindent\textbf{Eventual confinement and absorbing sets.} A subset $\mathcal A\subseteq U$ is called an absorbing set for a given evolution sequence if there exists $p^\star$ such that for all $p\ge p^\star$, the entities under consideration lie in $\mathcal A$ and never leave it along the subsequent sequence. When $\mathcal A$ is a fiber $F_\alpha(K)$, we may say the sequence exhibits eventual confinement inside the fiber of concept $K$.

\subsection{The Principle of Foundational Stability}

\paragraph{Background and Core Significance.}

A stable readout can be maintained by a persistent relation among changing quantities. A convergent coordinate is one possible anchor. If none converges, the Stability Theorem places at least two ambient cluster points in the joint confinement set $C_c^\star$. We read a convergent coordinate as a foundational anchor: a determinant that itself settles. Our interpretation is that long-lived stable structures in cognition rest either on such anchors or on persistent relations among changing quantities. The theorem establishes these two alternatives; it does not say which one a given system uses, and both can hold together. Mutual cancellation: for $f(v_1,v_2)=(v_1+v_2)/2$ with $v_1(p)=0.5+0.5\sin p$ and $v_2(p)=0.5-0.5\sin p$, each coordinate oscillates but $f\equiv 0.5$ is perfectly stable. Here $v_2(p)=1-v_1(p)$ at every step. The sum stays one and the mean stays at $1/2$. The same averaging rule gives a changing mean for pairs whose sum varies.

\paragraph{Topological convention.}
We fix an embedding $D\subseteq[0,1]^n$ and endow $D$ with its subspace topology. Unless otherwise stated, closures $\overline{A}$ appearing in $C_c^\star$ are taken in the ambient space $[0,1]^n$. The $\omega$-limit set collects cluster points in the ambient closure $\overline D$ \cite{HSD2012}.

\paragraph{Set-up.}
Let $f:D\to[0,1]$ be non-constant. For each index $p$, define the determining vector
\[
\mathbf v_p=\bigl(S^{(p)}(E_1,E_1'),\ldots,S^{(p)}(E_n,E_n')\bigr)\in D,\qquad y_p:=f(\mathbf v_p).
\]
A sequence $y_p$ is called stable if there exists $c\in[0,1]$ such that for every $\epsilon>0$ there is $p^\ast$ with $|y_p-c|\le\epsilon$ for all $p\ge p^\ast$. For $c\in[0,1]$ and $\epsilon>0$, define the level-set tube $N_\epsilon(c):=\{x\in D:\ |f(x)-c|\le\epsilon\}$ and the closed level set
\[
C_c^\star:=\bigcap_{m=1}^\infty \overline{\{x\in D:\ |f(x)-c|\le 1/m\}}.
\]
The $\omega$-limit set is $\omega(\mathbf v):=\{x\in \overline D:\ \exists\ p_k\to\infty\ \text{with}\ \mathbf v_{p_k}\to x\}$. A coordinate sequence $\pi_i(\mathbf v_p)$ is called an anchor if it converges.

\begin{theorem}[Stability Theorem]\label{thm:stability}
Assume $y_p=f(\mathbf v_p)$ is stable with limit $c$. Then:
\begin{enumerate}
\item For every $\epsilon>0$ there exists $p_\epsilon$ such that $\mathbf v_p\in N_\epsilon(c)$ for all $p\ge p_\epsilon$. Hence $\omega(\mathbf v)\subseteq C_c^\star$.
\item If $f$ is continuous on $D$, then $\omega(\mathbf v)\cap D\subseteq f^{-1}(\{c\})$. In particular, the full inclusion holds if every cluster point belongs to $D$.
\item Either some coordinate $\pi_i(\mathbf v_p)$ converges, or $\omega(\mathbf v)$ contains at least two distinct points and is contained in $C_c^\star$.
\end{enumerate}
\end{theorem}

\begin{lemma}[Separation Criterion]
Let $A,B\subseteq D$ and $\delta>0$ be such that $\mathbf v_p$ visits $A$ and $B$ infinitely often, and $|f(a)-f(b)|\ge\delta$ for all $a\in A$, $b\in B$. Then $y_p=f(\mathbf v_p)$ does not converge.
\end{lemma}

\begin{proof}
Choose subsequences along visits to $A$ and to $B$. Their images under $f$ differ by at least $\delta$, so $\limsup y_p-\liminf y_p\ge\delta>0$, which rules out convergence.
\end{proof}

\begin{lemma}\label{lem:unique-cluster}
If $\{\mathbf v_p\}$ has a unique cluster point $x^\star$ in the ambient compact space $[0,1]^n$ (hence also in $\overline D$), then $\mathbf v_p\to x^\star$, so each coordinate $\pi_i(\mathbf v_p)$ converges. Consequently, if no coordinate converges, then the set of ambient cluster points contains at least two distinct points.
\end{lemma}

\begin{proof}
If $\mathbf v_p\nrightarrow x^\star$, there exist $\eta>0$ and infinitely many $p$ with $\|\mathbf v_p-x^\star\|\ge\eta$. By compactness of $[0,1]^n$, this subsequence has a convergent sub-subsequence with limit $y\neq x^\star$, contradicting uniqueness.
\end{proof}

\begin{proof}[Proof of Theorem~\ref{thm:stability}]
Let $y_p=f(\mathbf v_p)$ and assume $y_p\to c$. Then for every $\epsilon>0$ there exists $p_\epsilon$ such that $|f(\mathbf v_p)-c|\le\epsilon$ for all $p\ge p_\epsilon$, i.e., $\mathbf v_p\in N_\epsilon(c)$ eventually. Hence any cluster point $x$ of $\{\mathbf v_p\}$ lies in every $\overline{N_{1/m}(c)}$, and thus $\omega(\mathbf v)\subseteq C_c^\star$. If $x\in\omega(\mathbf v)\cap D$, choose $\mathbf v_{p_k}\to x$. Continuity on $D$ gives $f(x)=\lim_k f(\mathbf v_{p_k})=c$. Continuity applies at each such point. For item (3), if some coordinate $\pi_i(\mathbf v_p)$ converges, we are in the first alternative of item (3). Otherwise, by Lemma~\ref{lem:unique-cluster}, the ambient cluster set contains at least two distinct points, and by the first paragraph $\omega(\mathbf v)\subseteq C_c^\star$, which is the second alternative of item (3).
\end{proof}

For $D=(0,1]$, $f(x)=x$ and $v_p=1/(p+1)$, the only ambient cluster point is $0$, outside the domain of $f$. Taking $D=[0,1]$ instead puts the cluster point in the domain. The same sequence eventually lies in each tube $N_\epsilon(0)$, even though no term reaches the exact zero level.

Without continuity of $f$ the set $C_c^\star$ can be the whole domain, so the theorem restricts the inputs mainly when the readout is regular.

\paragraph{Equivalent tail characterization.}\label{para:equiv-tail}
For sequences $y_p=f(\mathbf v_p)$, $y_p\to c$ if and only if for every $\epsilon>0$ there exists $p_\epsilon$ such that $\mathbf v_p\in N_\epsilon(c)$ for all $p\ge p_\epsilon$.

\begin{proof}
If $y_p\to c$, then for every $\epsilon>0$ there exists $p_\epsilon$ with $|f(\mathbf v_p)-c|\le\epsilon$ for all $p\ge p_\epsilon$, hence $\mathbf v_p\in N_\epsilon(c)$. Conversely, if for every $\epsilon>0$ there exists $p_\epsilon$ with $\mathbf v_p\in N_\epsilon(c)$ for all $p\ge p_\epsilon$, then $|f(\mathbf v_p)-c|\le\epsilon$ for all large $p$, so $y_p\to c$.
\end{proof}

\begin{corollary}[Contrapositive Test]\label{cor:contrapositive}
If there exists $\epsilon_0>0$ such that $\mathbf v_p\notin N_{\epsilon_0}(c)$ for infinitely many $p$, then $y_p$ cannot be stable at $c$.
\end{corollary}

\paragraph{Sufficient tail condition.}\label{para:sufficient-tail}If, for every $\epsilon>0$, there exists $p_\epsilon$ such that $\mathbf v_p\in N_\epsilon(c)$ for all $p\ge p_\epsilon$, then $y_p\to c$ and $\omega(\mathbf v)\subseteq C_c^\star$. For $f(x)=x$, the sequence alternating between zero and one stays in $N_1(0)$ but does not converge. One tube is insufficient. In this sense, an absorbing set is any set that the sequence eventually enters and then remains in.

\begin{proof}
The assumption implies $|f(\mathbf v_p)-c|\le\epsilon$ for all large $p$ and every $\epsilon>0$, hence $y_p\to c$. Any cluster point must lie in every $\overline{N_{1/m}(c)}$, so $\omega(\mathbf v)\subseteq C_c^\star$.
\end{proof}

\paragraph{Illustrative Examples.}{} Saturation: for $f(x)=\sigma(w^\top x+b)$ with a bounded monotone activation $\sigma$, if $w^\top x+b$ remains in a saturated regime, then $f(x)$ is nearly constant despite input jitter.

\section{Generative Dynamics and Intelligence}\label{sec:generative}

\begin{definition}[Generative Operator]\label{def:gen-op}
A \emph{generative operator}, denoted $G$, is a function that maps a current set of entities to a new set:
\[
E_{\text{new}} = G(X_p).
\]
The system then undergoes discrete evolution according to the update rule $X_{p+1}=X_{p}\cup E_{\text{new}}$, yielding an updated pair $Z_{p+1}=(X_{p+1},S^{(p+1)})$.
\end{definition}

In physics, energy is the ability to do work. Analogously, our prior work \cite{Ng2025} gave an informal definition of intelligence: ``intelligence is the ability, given entities exemplifying a concept, to generate entities exemplifying the same concept.'' This conceptual work is translated into the physical world by the action of a generative operator. For example, a large language model that, when given the conceptual requirement ``a Python function to efficiently sort a list of numbers'', generates the correct and functional source code.

Within the Similarity Field Theory framework, this principle is expressed formally. The premise, ``given entities exemplifying a concept,'' corresponds to a system-state $Z_p$ that contains a subset of entities $X_K \subseteq X_p$, where every entity in $X_K$ belongs to a fiber $F_{\alpha}(K)$.

\textbf{Intelligence with respect to a concept $K$} is the capacity of a generative operator $G$ to produce new entities $E' \in G(X_p)$ that also belong to the fiber $F_{\alpha}(K)$ by satisfying the condition:
\[
S(E', K) \ge \alpha.
\]

\noindent\textbf{Example (ChatGPT, Ghibli style).} Let $K$ denote the concept ``Studio Ghibli--style image.'' Provide a reference context of Ghibli-style images and prompt ChatGPT to produce new frames under the same context. Using a fixed similarity function $S(\cdot,K)$ (e.g., style embeddings or calibrated ratings) to score each output, if a substantial portion of the outputs satisfy $S\ge \alpha$ (high coverage) and the average score is high (high fidelity), then ChatGPT is $\alpha$-intelligent with respect to the concept ``Ghibli-style image.''

We can therefore quantify intelligence by within-output coverage (fraction in $F_\alpha(K)$) and fidelity (mean $S(E',K)$). Intelligence in this sense is graded, not pass-or-fail. For a stated input and a stated batch of outputs it is reported by coverage and fidelity at a positive threshold $\alpha$; any minimum coverage required for a verdict is a convention of the evaluator and should be stated with the result. Here a new entity means one that is not already in $X_p$.

\subsection{Learning}\label{sec:learning}
Learning is the process by which a system's generative operator $G$ or its similarity field $S^{(p)}$ is modified over the sequence (i.e., as $p$ increases) to improve its intelligence with respect to a concept $K$. For each step $p$, take a nonempty finite batch $Y_p\subseteq G(X_p)$ chosen by a rule fixed in advance, for example all of $G(X_p)$. Its mean fidelity is $L_p=\frac{1}{|Y_p|}\sum_{e\in Y_p}S^{(p)}(e,K)$, and the learning rate per step is $\Delta L_p=L_{p+1}-L_p$.

Two ways of changing a field are described in Section~\ref{sec:observers}.

Definition~\ref{def:entity} allows an entity to be a similarity relation, and the same holds for a generative operator. Let $K'$ be a concept whose exemplars are operators that an observer scores as intelligent with respect to some concept. An operator that, given such operators, generates further operators in the fiber of $K'$ is intelligent with respect to $K'$ by the same definition. Learning is then generation one level up, where what is produced is an operator or a field. The definition is the same at every level, and Theorem~\ref{thm:breadth} applies to a batch of generated operators exactly as it does to any other entities.

\subsection{Creativity as Re-Contextualization}\label{sec:creativity-sft}
Furthermore, our framework provides a formal perspective on creativity. This is not treated as creation ex nihilo, but as the formation of a new concept initiated by the re-contextualization of an existing entity. This occurs when an entity $E$, already a member of a source concept $K_{\text{source}}$, is discovered to also possess high similarity to a different, nascent target concept $K_{\text{target}}$. Prior to this discovery, the fiber of the target concept, $F_{\beta}(K_{\text{target}})$, might have contained no entity of the system other than $K_{\text{target}}$ itself. The act of recognizing that $S(E, K_{\text{target}}) \ge \beta$ gives this fiber its first member other than the concept itself, establishing $K_{\text{target}}$ as a valid concept and positioning $E$ as its foundational instance. The history of WD-40 illustrates this. Its formula ($E$) was an entity in the fiber of ``aerospace anti-corrosion agent'' ($K_{\text{source}}$). The discovery of its utility in household applications was a recognition of its high similarity to the nascent concept of a ``general-purpose lubricant'' ($K_{\text{target}}$), making it an early entity to populate that conceptual fiber for household use. At an abstract level, creativity can also be characterized as re-contextualization: an entity in the fiber of a source concept is recognized as crossing the threshold for a target concept, thereby giving the new fiber its first member other than the concept itself.

\par\begingroup
\subsection[Transitivity, breadth and the fidelity--breadth theorem]{Transitivity, breadth and the fidelity--breadth theorem}\label{sec:breadth}

Section~\ref{sec:generative} defined intelligence as the ability, given entities in the fiber of a concept, to generate entities in the same fiber. We now ask how many effectively different things such generation can produce. Throughout, the input contains at least one supplied exemplar, and we identify the particular outputs being assessed. Membership tests only one directed score per output. To infer a relation between two outputs through $K$, the direction of both links matters.

\begin{definition}[Transitivity degree]\label{def:transitivity}
On a nonempty finite set $X$, write
\[
  \tau(A,B,C)=\min\!\left\{1,\frac{S(A,C)}{S(A,B)S(B,C)}\right\}
\]
for distinct $A,B,C$ with positive denominator. Let $T=T(S;X)$ be the minimum over these triples. For $K\in X$, let $T_K$ be the minimum restricted to triples with middle vertex $K$. Both minima have value one on an empty index set.
\end{definition}
Thus $0\le T\le T_K\le1$. The local degree is the largest $t\in[0,1]$ for which $S(A,C)\ge tS(A,K)S(K,C)$ on $X$. Repeated vertices and zero products satisfy the inequality automatically. In particular $T=1$ is equivalent to product transitivity. Reversing a triple can change its ratio. The global degree is the product specialization of a fuzzy transitivity degree \cite{Belohlavek2007}. A minimum using the smaller of the two link scores would express a different assumption.

To count effectively different outputs we need both their frequencies, if specified, and their mutual scores.
\begin{definition}[Similarity-sensitive diversity, after Leinster and Cobbold]\label{def:diversity}
For $n\ge1$ distinct entities $E_1,\ldots,E_n$, let $\scoremat_{ij}=S(E_i,E_j)$, let $\weights$ be a probability vector, and put $I=\{i:p_i>0\}$ and $z_i=(\scoremat\weights)_i$. For $0\le q<\infty$, $q\ne1$, set
\[
 D_q(\weights)=\left(\sum_{i\in I}p_i z_i^{q-1}\right)^{1/(1-q)} .
\]
The limiting orders give $D_1=\exp(-\sum_{i\in I}p_i\log z_i)$ and
$D_\infty=1/\max_{i\in I}z_i$. In particular $D_2(\weights)=1/(\weights^{\mathsf T}\scoremat\weights)$.
\end{definition}
Leinster and Cobbold's definition already permits asymmetric $\scoremat$ \cite[p.~479 and Appendix~A, Example~A6]{LeinsterCobbold2012}. For a fixed probability vector $\weights$, reflexivity gives $z_i=p_i+\sum_{j\ne i}\scoremat_{ij}p_j\ge p_i>0$ on its support $I$. Hence
\[
 \log_2 D_1(\weights)=-\sum_{i\in I}p_i\log_2 z_i
 \ \le\ -\sum_{i\in I}p_i\log_2 p_i=H_{\mathrm{bits}}(\weights).
\]
Equality holds precisely when every off-diagonal score between labels of positive probability is zero. With identity similarity one recovers the Hill numbers \cite{Hill1973} and $D_1=2^{H_{\mathrm{bits}}}=\exp(H_{\mathrm{nats}})$. Base two expresses Shannon entropy in bits \cite{Shannon1948}.

\textit{Remark (The entropy gap).} Suppose the scores between distinct labels are at most $\gamma\in[0,1]$. Let $n$ be the number of labels and $k=|I|$ the number with positive weight. Then
\[
 H_{\mathrm{bits}}(\weights)-\log_2[1+(k-1)\gamma]
 \le\log_2 D_1(\weights)\le H_{\mathrm{bits}}(\weights).
\]
The lower bound also holds with $n$ in place of $k$. Equivalently, the effective count $D_1$ lies between $2^{H_{\mathrm{bits}}}/[1+(k-1)\gamma]$ and $2^{H_{\mathrm{bits}}}$. Uniform weights on the support and constant off-diagonal score $\gamma$ there attain the lower endpoint. For $(n,\gamma)=(10,0.01),(10,0.05),(100,0.01)$, the full-label factors are $1.09$, $1.45$ and $1.99$, respectively. This bounds the overstatement of log diversity by label entropy in bits. Relative agreement with the identity field gives $r=0$ as soon as a positive off-diagonal score is compared with zero.

\textit{Proof.} On $I$, we have $p_i\le z_i\le(1-\gamma)p_i+\gamma$. The entropy gap is $\sum_{i\in I}p_i\log_2(z_i/p_i)\ge0$. By concavity of the logarithm it is at most $\log_2(\sum_{i\in I}z_i)\le\log_2[1+(k-1)\gamma]$. In the stated equality case every supported ratio $z_i/p_i$ equals $1+(k-1)\gamma$.

A hundred equally frequent, separately labelled strings have Shannon entropy $\log_2 100$ bits whether they are paraphrases of one idea or descriptions of a hundred different ideas. Measuring the similarities among the strings can distinguish the two cases. Grouping strings by meaning can also change their label entropy \cite{Kuhn2023}. Choosing the identity matrix recovers ordinary label entropy after taking the logarithm of order-one diversity.

\begin{definition}[Breadth of a generated batch]\label{def:breadth}
Fix an input $X_p$ containing at least one exemplar in $F_\alpha(K)$. Specify a nonempty batch of distinct generated outputs $Y=\{E_1,\ldots,E_n\}\subseteq G(X_p)\cap F_\alpha(K)$. Its breadth is
\[
 B(Y)=\max_{\weights\in\Delta(Y)}D_2(\weights),
\]
where $\Delta(Y)$ is the probability simplex on this batch. State whether $Y$ is the whole output or its qualifying subset. For this input and this choice of $Y$, breadth at least $b_0$ means $B(Y)\ge b_0$.
\end{definition}
Observed frequencies give one value of $D_2$. Breadth optimizes over the weights, so the scores on the specified batch determine $B(Y)$ without an observed probability distribution. The minimum of the continuous denominator exists, and
$\sum_i p_i^2\le\weights^{\mathsf T}\scoremat\weights\le1$ gives $1\le B\le n$. If all mutual scores are one, $B=1$. Scores close to one against a single input alone impose no such restriction. For example, four outputs can each score $0.99$ in both directions with that input and zero with one another. Their breadth is four. The generative definition itself sets no minimum breadth. A scribe who copies a supplied text onto a fresh sheet produces a new entity. By reliably making copies that exemplify $K$, the scribe meets the intelligence criterion for $K$, even when every pair of copies scores one and the batch has breadth one. Coverage counts the fraction that qualify and fidelity averages their scores towards $K$. Breadth compares the copies with one another, while creativity concerns recognizing an entity in a new conceptual setting (Section~\ref{sec:creativity-sft}).

\Needspace{5\baselineskip}
A lower bound on scores between different outputs imposes a diversity ceiling at every order.
\begin{lemma}[Diversity ceiling]\label{lem:diversity}
If all off-diagonal scores are at least $\beta\in(0,1]$, then $D_q(\weights)\le1/\beta$ for all $\weights$ and all $q\in[0,\infty]$. If all off-diagonal scores equal $\beta$ and $\weights$ is uniform, then $D_q=n/[1+(n-1)\beta]$ for every order.
\end{lemma}
\begin{proof}
Each weighted row score $z_i=(\scoremat\weights)_i$ is at least $\beta$, including when the batch has one member. For $1<q<\infty$, the sum $\sum_{i\in I}p_i z_i^{q-1}$ is at least $\beta^{q-1}$. Raising both sides to the negative power $1/(1-q)$ gives $D_q\le1/\beta$. For $0\le q<1$, the sum is at most $\beta^{q-1}$ and the power $1/(1-q)$ is positive, giving the same bound. At $q=1$, use $\sum_{i\in I}p_i\log z_i\ge\log\beta$. At $q=\infty$, use $\max_{i\in I} z_i\ge\beta$. With uniform weights and constant off-diagonal scores every $z_i=\beta+(1-\beta)/n$, so all orders have the stated value.
\end{proof}

\Needspace{12\baselineskip}
\begin{theorem}[Fidelity--breadth theorem]\label{thm:breadth}
Let $X$ be a finite set containing a concept $K$ and a batch $Y=\{E_1,\ldots,E_n\}$ of $n\ge1$ distinct outputs, all different from $K$. For a threshold $0<\alpha\le1$, suppose
\[
 S(E_i,K)\ge\alpha \quad\text{and}\quad S(K,E_i)\ge\alpha
 \qquad (1\le i\le n).
\]
Compute $T_K$ on $X$, assume $T_K>0$, and put $\beta=T_K\alpha^2$. Every ordered pair of distinct outputs has $S(E_i,E_j)\ge\beta$. For every probability vector $\weights$ on $Y$ and every $0\le q\le\infty$,
\[
 D_q(\weights)\le\frac{n}{1+(n-1)\beta}\le\frac1\beta .
\]
Maximizing at $q=2$ gives the finite bound for $B(Y)$. For fixed $n$ the bound is sharp at every order, and its limit $1/\beta$ is the sharp size-independent bound.
\end{theorem}
\begin{proof}
The two directed links $E_i\to K\to E_j$ give
$S(E_i,E_j)\ge T_K S(E_i,K)S(K,E_j)\ge\beta$. The unit diagonal gives $z_i\ge\beta+(1-\beta)p_i>0$. Set $g(u)=u/(\beta+(1-\beta)u)$. Since $g''(u)=-2\beta(1-\beta)/(\beta+(1-\beta)u)^3\le0$ on $[0,1]$, Jensen's inequality with equal weights gives
\[
 D_0(\weights)=\sum_i p_i/z_i\le\sum_i g(p_i)
 \le n\,g(1/n)=\frac{n}{1+(n-1)\beta}.
\]
The weighted power-mean inequality gives $D_q\le D_0$: on the positive support, $1/D_q$ is the power mean of the $z_i$ of order $q-1$, with the geometric mean at $q=1$ and the maximum at $q=\infty$. The harmonic mean occurs at $q=0$. Finally $n\beta\le1+(n-1)\beta$.
For $n\ge2$ and any $0<t\le1$ take $X=\{K,E_1,\ldots,E_n\}$, give each incident link to or from $K$ score $\alpha$, and set every off-diagonal output score to $t\alpha^2$. Then $T_K=t$ and all uniformly weighted row scores are $\beta+(1-\beta)/n$, so the same weights give equality at every order. As $n$ increases with $t$ and $\alpha$ fixed, these breadths tend to $1/(t\alpha^2)$. On a singleton all orders equal one, the finite bound.
\end{proof}
Figure~\ref{fig:breadth} shows the size-independent bound at four values of $T_K$.

\begin{figure}[H]
\centering
{\color{black}
\begin{tikzpicture}
\begin{axis}[width=0.96\linewidth,height=5.9cm,xmin=.50,xmax=1,
    ymin=.9,ymax=44,ymode=log,log basis y=10,
    xlabel={Two-sided fidelity threshold $\alpha$},ylabel={Breadth ceiling},
    xtick={.5,.6,.7,.8,.9,1},ytick={1,2,4,10,20,40},
    yticklabels={1,2,4,10,20,40},grid=major,
    grid style={black!12},tick label style={font=\small},label style={font=\small},
    legend style={draw=none,font=\footnotesize,at={(.985,.98)},anchor=north east,fill=white},
    legend columns=2,domain=.5:1,samples=55,no markers]
  \addplot[very thick,black] {1/(x*x)}; \addlegendentry{$T_K=1$}
  \addplot[thick,blue!70!black,dashed] {2/(x*x)}; \addlegendentry{$T_K=0.5$}
  \addplot[thick,teal!85!black,dashdotted] {4/(x*x)}; \addlegendentry{$T_K=0.25$}
  \addplot[thick,violet!80!black,densely dotted] {10/(x*x)}; \addlegendentry{$T_K=0.1$}
\end{axis}
\end{tikzpicture}}
\caption[Two-sided fidelity and the breadth ceiling]{Fidelity--breadth theorem (two-sided fidelity). At full transitivity through $K$ ($T_K=1$) and high fidelity, the ceiling is close to one effective kind. It rises as $T_K$ falls. The vertical scale is logarithmic.}
\label{fig:breadth}
\end{figure}
\FloatBarrier

Conversely, an observed breadth gives a necessary condition on transitivity through the concept.
\begin{corollary}[Transitivity required by observed breadth]\label{cor:necessary}
If $S(E_i,K),S(K,E_i)\ge\alpha>0$ for each member of the batch in Theorem~\ref{thm:breadth}, breadth at least $b_0>0$ requires $T_K\le1/(b_0\alpha^2)$.
\end{corollary}
\begin{proof}For $T_K>0$ rearrange the size-independent bound. At $T_K=0$ the inequality is automatic and $B\le n$ still holds.\end{proof}
At $\alpha=0.9$ the ceilings for $T_K=1,0.5,0.1$ are approximately $1.23,2.47,12.35$. A low $T_K$ permits greater breadth. It need not produce it. A positive global $T$ gives a weaker valid bound. In particular, breadth greater than $1/\alpha^2$ requires $T_K<1$ under the same two-sided threshold. At $\alpha=0.5$ and $T_K=1$ the sharp finite ceiling is $4n/(n+3)$. The result uses the product form of Definition~\ref{def:transitivity}.

Low breadth is not low intelligence. Intelligence with respect to $K$ is reported by coverage and fidelity (Section~\ref{sec:generative}); breadth separately measures how different the qualifying outputs are from one another. A semiconductor fab that manufactures many chips of one design, each within tight tolerances, generates new entities in the fiber of the concept ``this chip design'' at very high fidelity. Its intelligence with respect to that concept is high although the breadth of its output is close to one. Theorem~\ref{thm:breadth} explains the pattern: with two-sided fidelity close to one and high transitivity through $K$, the breadth ceiling is close to one, so near-identical output is what such fidelity requires rather than a sign that it is absent. Whether the same system is intelligent with respect to a broader concept, such as ``working processor,'' is a separate question with its own measurements.

The distinction from one-sided intelligence can be exact. Give every output score one towards $K$, set all reverse scores from $K$ to zero and all scores between distinct outputs to zero. Add a supplied exemplar that scores one towards $K$, zero from $K$, and zero in either direction with each output. Every positive two-step path is supported by a diagonal edge or ends at $K$ along a single nontrivial edge, so the field is product-transitive. The outputs lie in $F_1(K)$ and their breadth is $n$. For two-sided data the same theorem also bounds a batch enlarged to include given exemplars, provided each included entity meets both thresholds. In particular the scores between a given exemplar and a new output are at least $T_K\alpha^2$ in both directions when they meet those conditions. The ceiling follows directly from the local inequality and the diversity lemma. Its force is the sharpness and the distinction between the two experimental questions.

In Definitions~\ref{def:state-seq} and~\ref{def:gen-op}, successive acts of generation can add new exemplars of $K$ to the system's world. Take any nonempty finite set of these outputs distinct from $K$ and score them together in the current field. If each scores at least $\alpha>0$ in both directions with $K$ and the local degree $T_K$ on this set together with $K$ is positive, the maximum of $D_2(\weights)$ over all probability weights on this pooled set is at most $1/(T_K\alpha^2)$, however many steps produced them.

\textit{Remark (Saturation).} Fix a field $S$, a concept $K$ and $0<\alpha\le1$. Let $P_1\subseteq P_2\subseteq\cdots$ be nonempty finite pools. Write $b_m=\max_{w\in\Delta(P_m)}D_2(w)$ and let $t_m$ be $T_K$ on $P_m\cup\{K\}$. Then $b_m$ is nondecreasing and $t_m$ is nonincreasing. Suppose every member other than $K$ scores at least $\alpha$ in both directions with $K$. If $t=\inf_m t_m>0$, then $b_m\le C=1/(t\alpha^2)$ for every $m$. Thus $b_m$ converges and $b_{m+1}-b_m\to0$, whatever the generative operator. For each $\varepsilon>0$, at most $(C-1)/\varepsilon$ steps have increment at least $\varepsilon$. Under the real unit-vector hypotheses of Corollary~\ref{cor:gram-breadth} on each pool together with $K$, one-sided fidelity $S(E,K)\ge\alpha$ suffices for the same conclusions with $C=1/\alpha^2$, with no condition on transitivity.

\textit{Proof.} Extending weights by zero preserves their quadratic denominator, and enlarging the set adds constraints to the local degree. The local inequality gives the pairwise floor $t_m\alpha^2$ used in Theorem~\ref{thm:breadth}, including when $K$ is pooled, by reflexivity. In the unit-vector case apply Corollary~\ref{cor:gram-breadth} instead. The resulting bounded monotone sequence converges, so successive differences tend to zero. A point mass gives $b_1\ge1$. Summing the nonnegative increments up to any step gives at most $C-1$, proving the count.

A changing field can reverse monotonicity even with the pool fixed. Take two outputs and keep all four directed links between them and a separate $K$ at $1/2$. Change the two scores between the outputs from $1/2$ to $1$. For weights $u,1-u$, the quadratic denominator changes from $1-u(1-u)$ to $1$, so breadth falls from $4/3$ to $1$. Two pooled entities are minimal, since a singleton always has breadth one.

\paragraph{A restriction imposed by unit-vector scores.}
For comparison, consider the case in which similarities are inner products of unit vectors. Suppose a common real inner-product representation is used for both the output--concept scores and the output--output scores. A field of this kind can have $T=0$: two orthogonal unit vectors each have inner product $1/\sqrt{2}$ with their normalized sum. Nevertheless geometry constrains which sets of scores can coexist.

\begin{proposition}[Three inner products]\label{prop:gram-real}
If a finite score matrix is symmetric and positive semidefinite (PSD), then
\[
 S(A,C)\ge S(A,B)S(B,C)
 -\sqrt{(1-S(A,B)^2)(1-S(B,C)^2)} .
\]
\end{proposition}
\begin{proof}
The matrix is a Gram matrix of real unit vectors, by its unit diagonal. Resolve $u_A$ and $u_C$ into projections on $u_B$ and perpendicular residuals. The first terms contribute $S(A,B)S(B,C)$, and the residual inner product is bounded below by the negative product of their norms.
\end{proof}
A violation rules out such a representation in every dimension. Symmetry also precludes the hypothesis of the Incompatibility Theorem.

When similarities are inner products of real unit vectors, membership in the fiber already gives a breadth ceiling.
\begin{corollary}[Dimension-free breadth ceiling]\label{cor:gram-breadth}
Suppose the kernel on $Y\cup\{K\}$ is symmetric PSD, has unit diagonal and takes values in $[0,1]$. If $S(E_i,K)\ge\alpha>0$ for every $E_i\in Y$, then $D_2(\weights)\le1/\alpha^2$ for every $\weights$. The same bound holds for $B(Y)$.
\end{corollary}
\begin{proof}
Let $m=\sum_i p_i u_i$. Then
\[
 \weights^{\mathsf T}\scoremat\weights=\|m\|^2
 \ge \langle m,u_K\rangle^2
 =\left(\sum_i p_i S(E_i,K)\right)^2\ge\alpha^2 .
\]
To attain the size-independent bound, choose $n\ge2$ and take unit regular-simplex vertices $w_i\perp u_K$ with $\sum_iw_i=0$, and put
$u_i=\alpha u_K+\sqrt{1-\alpha^2}\,w_i$. Uniform weights give $m=\alpha u_K$.
The off-diagonal score is $\alpha^2-(1-\alpha^2)/(n-1)$. Taking $n$ sufficiently large keeps it nonnegative for each fixed positive threshold. Two antipodal $w_i$ suffice for $\alpha\ge1/\sqrt{2}$.
\end{proof}
The norm $\|\sum_i p_i u_i\|$ is the weighted mean resultant length of directional statistics. The standard Gram identity gives $D_2(\weights)$ as its reciprocal square. What is specific here is its link to the definition of intelligence and to Theorem~\ref{thm:breadth}. Applied to generated exemplars it gives the same size-independent ceiling as $T_K=1$ in Theorem~\ref{thm:breadth}, with different finite-size equality cases. The ceilings at $\alpha=0.90,0.95,0.99$ are approximately $1.235,1.108,1.020$ respectively. This statement concerns order two and a common scoring field. In particular, observing $D_2>1/\alpha^2$ for even one weighting excludes an exact unit-vector representation for those scores, in any dimension.

The same fact can be counted in whole entities.
\begin{corollary}[Counting separated exemplars]\label{cor:packing}
Under the Gram hypotheses of Corollary~\ref{cor:gram-breadth}, suppose $n$ distinct entities score at least $\alpha>0$ towards $K$ and each distinct pair scores at most $\gamma$, where $0\le\gamma<\alpha^2$. Then
\[
 n\le\frac{1-\gamma}{\alpha^2-\gamma}.
\]
\end{corollary}
\begin{proof}
Taking the unweighted sum gives
\[
 n^2\alpha^2\le\left\langle \sum_i u_i,u_K\right\rangle^2
 \le\left\|\sum_i u_i\right\|^2
 =n+\sum_{i\ne j}S(E_i,E_j)\le n+n(n-1)\gamma.
\]
Divide by $n$ and collect terms. When the bound is an integer $n\ge2$, the simplex construction above gives off-diagonal score exactly $\gamma$.
\end{proof}
This is a spherical-code argument of the simplex-bound type \cite{Rankin1955}. For $\alpha=0.9$ this gives:
\begin{center}\small
\begin{tabular}{@{}lrrrrr@{}}
\toprule
Maximum pair score $\gamma$ & $0.5$ & $0.7$ & $0.715$ & $0.75$ & $0.8$\\
Maximum number of entities & $1$ & $2$ & $3$ & $4$ & $20$\\
\bottomrule
\end{tabular}
\end{center} A general reflexive field admits arbitrarily many examples at fixed scores $\alpha$ to and from $K$ and $\gamma$ between distinct examples. Once the integer cap in a unit-vector field is reached, any further qualifying output must score above $\gamma$ with some counted entity. Supplied exemplars count too if they are included in that set. The limit binds representations by inner products of unit vectors, including kernels with unit diagonal. Models that score a pair directly are not bound by it. How tight the limit is for a particular embedding model depends on the values of $\alpha$ and $\gamma$ that the model actually produces, which is empirical.

\paragraph{A consequence for evaluation.} When prompt fidelity and output diversity are evaluated in the same normalized embedding space, the evaluator constrains which scores can coexist. For a fixed prompt and $n>1$ equally weighted outputs with mean prompt cosine $F$, the mean cosine distance between distinct outputs is at most $n(1-F^2)/(n-1)$. For three outputs at $F=0.9$ this ceiling is $0.285$. Reporting the ceiling alongside the scores shows how close the measured separation lies to its geometric limit. The bound applies only when both scores use one embedding space and one fixed reference; an evaluation that measures fidelity and diversity with different models is not bound by it. The inequality is the computation in the proof of Corollary~\ref{cor:packing} with the mean in place of the threshold. For text-based domain adaptation, StyleDomain uses a shared CLIP embedding space to measure mean cosine similarity to one target description and mean pairwise cosine distance between generated images \cite[supplement, Sec.~A.2.3]{Alanov2023}. In that study's reported text-based results the mean prompt cosine lies between 0.18 and 0.31, where this ceiling is above 0.9 and does not bind; it becomes restrictive only at high fidelity.

\paragraph{What the six comparisons measure.}
Imagine a tabby photograph $A$, a cartoon $B$, a sculpture $C$ and a single concept entity $K$, cat. One observer scores $A$ against $K$, $B$ against $K$ and $C$ against $K$. The other three comparisons are $A$ with $B$, $A$ with $C$ and $B$ with $C$. All six must use the same relation and scale. If the first three scores are at least $0.9$, the average of the last three in a unit-vector model is at least $(3\cdot0.9^2-1)/2=0.715$. Hence at least one of those pairs exceeds $0.7$. This limits the number of mutually separated exemplars of \emph{one fixed concept}. If $G$ generated the objects, it limits the scores that a single representation can give that batch. The same calculation applies to three objects already on a table. The test asks whether a single assignment of four unit vectors can reproduce the six scores. A classifier's category probability and a separate picture-style score would measure two different relations.

More generally, $n\ge2$ examples with constant pair score $s$ have uniform diversity $n/[1+(n-1)s]$. This exceeds $1/\alpha^2$ precisely when $s<(n\alpha^2-1)/(n-1)$. At $\alpha=0.9$ the cutoff is $0.62$ for two examples, $0.715$ for three and approaches $0.81$ as $n$ grows. Wittgenstein's games and Rosch and Mervis's family-resemblance studies motivate asking how category members resemble one another in different respects \cite{Wittgenstein1953,RoschMervis1975}. The six cat scores are a thought experiment. It is not a measurement. Whether people give such scores on one calibrated scale is an empirical question. We present it as motivation and as a prediction for future work.
\par\endgroup

\par\begingroup
\subsection[Asymmetry, ranking and circulation]{Asymmetry, ranking and circulation}\label{sec:asymmetry}

A similarity field is directed: $S(A,B)$ and $S(B,A)$ may differ. This section separates the two forms that this asymmetry can take. A ranking of the entities can explain some such differences, and three-way comparisons detect the part that a ranking cannot explain. In this subsection $X$ is finite and nonempty and every score is strictly positive.

\begin{definition}[Cycle ratio]\label{def:cycle}
For $A,B,C\in X$, put
\[
 \rho(A,B,C)=\frac{S(A,B)S(B,C)S(C,A)}{S(A,C)S(C,B)S(B,A)} .
\]
\end{definition}
\Needspace{8\baselineskip}
A change of vertex weights can account for all directional preferences precisely in the following case.
\begin{theorem}[Cycle criterion]\label{thm:cycle}
On a finite nonempty set $X$ with $S(A,B)>0$ for every ordered pair, the following conditions are equivalent.
\begin{enumerate}
\item[(i)] There are weights $w(A)>0$ such that $w(A)S(A,B)=w(B)S(B,A)$ for all $A,B\in X$.
\item[(ii)] $\rho(A,B,C)=1$ for all $A,B,C\in X$.
\end{enumerate}
\end{theorem}
\begin{proof}
Under (i), $S(A,B)/S(B,A)=w(B)/w(A)$, and multiplying the three ratios around a triangle gives one. Conversely, assume (ii), fix $O\in X$ and put $w(A)=S(O,A)/S(A,O)$. The identity $\rho(O,A,B)=1$ gives
\[
  S(O,A)S(A,B)S(B,O)=S(O,B)S(B,A)S(A,O).
\]
Dividing by $S(A,O)S(B,O)>0$ gives (i).
\end{proof}
This is the complete positive case of Kolmogorov's cycle criterion \cite{Kolmogoroff1936,FKelly1979}. Dividing every row of $S$ by its sum gives a Markov transition matrix with the same cycle ratios, because the three row sums cancel.

\Needspace{8\baselineskip}
For judging directed scores, the next result assigns one score to each entity and one residual to each pair.
\begin{theorem}[Asymmetry decomposition]\label{thm:hodge}
Let $X$ have $n\ge3$ entities with strictly positive scores on $X\times X$. Set $Y(A,B)=\log S(A,B)-\log S(B,A)$ and fix any order on $X$.
\begin{enumerate}
\item[(a)] There are a unique function $s:X\to\mathbb R$ with $\sum_A s(A)=0$ and a unique antisymmetric $c$ with $\sum_B c(A,B)=0$ for every $A$ such that
\[
  Y(A,B)=s(B)-s(A)+c(A,B),\qquad s(A)=\frac1n\sum_B Y(B,A).
\]
\item[(b)] Put $E_c=\sum_{A<B}c(A,B)^2$. The energy splits as
$\sum_{A<B}Y(A,B)^2=\sum_{A<B}(s(B)-s(A))^2+E_c$.
\item[(c)] The triangle log ratios determine the residual energy via
\[
  nE_c=\sum_{A<B<C}\bigl(\log\rho(A,B,C)\bigr)^2 .
\]
\item[(d)] The residual $c$ vanishes exactly when all cycle ratios equal one. In that case $w(A)=\exp(s(A))$ satisfies the detailed-balance identity of Theorem~\ref{thm:cycle}(i).
\end{enumerate}
\end{theorem}
\begin{proof}
\textup{(a)} Summing the decomposition over $B$ forces $s(A)=-\frac1n\sum_B Y(A,B)$. Antisymmetry of $Y$ gives the displayed formula and $\sum_A s(A)=0$. Subtracting $s(B)-s(A)$ from $Y(A,B)$ now gives the unique residual, which is antisymmetric and has zero row sums.

\textup{(b)} The cross term $\sum_{A<B}(s(B)-s(A))c(A,B)$ is half the sum over ordered pairs. Both its terms vanish because $c$ has zero row and column sums. Expanding the square gives (b).

\textup{(c)} Write $d_{ABC}=\log\rho(A,B,C)$. The differences of $s$ telescope around the triangle, so $d_{ABC}=c(A,B)+c(B,C)+c(C,A)$. Summing over $C$ gives $\sum_C d_{ABC}=n\,c(A,B)$. Multiply by $c(A,B)$ and sum over $A<B$. For each $A<B<C$, the three edge terms contribute
\[
 d_{ABC}\bigl[c(A,B)+c(B,C)-c(A,C)\bigr]=d_{ABC}^2 ,
\]
since reversing the triangle changes the sign of $d$. This gives (c).

\textup{(d)} The sum of squares in (c) vanishes precisely when every cycle ratio equals one. With $c=0$, exponentiating $Y(A,B)=s(B)-s(A)$ gives the identity for $w$.
\end{proof}
This is the complete-graph Hodge decomposition \cite{Jiang2011} applied to log ratios. Equivalently, with $S_{\rm geom}(A,B)=\sqrt{S(A,B)S(B,A)}$,
\[
 S(A,B)=S_{\rm geom}(A,B)\exp\big((s(B)-s(A))/2\big)\exp\big(c(A,B)/2\big).
\]
If $\sum_{A<B}Y(A,B)^2>0$, the fraction $\kappa=E_c/\sum_{A<B}Y(A,B)^2$ lies in $[0,1]$. It is undefined at exact symmetry. The ranking records a bias towards higher-scoring entities. The residual measures how such scores fail to account for the observed directions. By part (d) and the cycle criterion, the row-normalized Markov chain is reversible exactly when $c=0$. Some directionality may belong to the world itself and some may arise only at the level of the observer. The framework accommodates both.

For example, $S(\mathrm{cat},\mathrm{animal})=0.9$ and the reverse score $0.1$ fit weights in ratio nine. For three entities, take forward scores $S(A,B)=S(B,C)=0.9$ and $S(A,C)=0.81$, and reverse scores $0.1,0.1,0.01$ respectively. Both cycle products equal $0.0081$. Alternatively, for a directed association cycle spring, summer, autumn, spring, assign the forward scores $0.9$ and the reverse scores $0.1$. Then $\rho=729$, $s=0$ and $\kappa=1$. These numbers illustrate two mechanisms. The relative contribution of circulation in measured judgments is open.

The transitivity degree that bounds output breadth also constrains directed cycles.
\begin{proposition}[A transitivity bound on circulation]\label{prop:cyclebound}
For three distinct entities let $f=S(A,B)S(B,C)S(C,A)$ and $b=S(A,C)S(C,B)S(B,A)$. Then $b\ge T^3 f^2$, or equivalently $\rho\le 1/(T^3 f)$.
\end{proposition}
\begin{proof}
Product transitivity with degree $T$ gives $S(A,C)\ge T S(A,B)S(B,C)$, $S(C,B)\ge T S(C,A)S(A,B)$ and $S(B,A)\ge T S(B,C)S(C,A)$. Multiplying uses each forward edge twice and gives $b\ge T^3 f^2$. All scores and $T$ are positive on the finite set, so division gives the bound for $\rho=f/b$.
\end{proof}
Thus transitivity limits both breadth within the fiber of a concept under two-sided fidelity (Theorem~\ref{thm:breadth}) and circulation (Proposition~\ref{prop:cyclebound}). The three-stage example attains equality with $T=10/81$. Even at $T=1$, the bound allows a cycle ratio as high as $1/f$.
\par\endgroup

\par\begingroup
\subsection[Observers, agreement, transfer and richness]{Observers, agreement, transfer and richness}\label{sec:observers}

Intelligence was defined relative to a similarity field, that is, relative to an observer. Any system that induces a similarity field on a set of entities is an observer in this sense; nothing below requires it to be human or to be a learner. We ask how much of what one observer accepts survives a change of observer. In this subsection $X$ is the same nonempty finite set for both. An observer is specified by the system and the response used to obtain its scores.

\begin{definition}[Agreement]\label{def:agreement}
Let $h(0,0)=1$ and $h(a,b)=\min(a,b)/\max(a,b)$ otherwise for $a,b\in[0,1]$. Define
\[
 r(S_1,S_2)=\min_{A,B\in X} h\big(S_1(A,B),S_2(A,B)\big).
\]
\end{definition}
If $r>0$, the fields have identical zero patterns and $rS_1\le S_2\le S_1/r$ entry by entry. Thus agreement on all scores controls agreement on the ones used to judge generation.

\begin{proposition}[Agreement is product-transitive]\label{prop:agreement}
On fields over $X$, $r$ is reflexive, symmetric and takes values in $[0,1]$. It equals one precisely for identical fields. Moreover,
\[
 r(S_1,S_3)\ \ge\ r(S_1,S_2)\,r(S_2,S_3).
\]
With $-\log 0=+\infty$, $d=-\log r$ is an extended metric. On positive fields it is the largest absolute difference between corresponding log scores.
\end{proposition}
\begin{proof}
The scalar ratio gives the first assertions. If either factor on the right vanishes, the inequality follows from $r\ge0$. Otherwise the three fields have the same zeros. On their common positive support $-\log h(a,b)=|\log a-\log b|$. The triangle inequality for the maximum of the coordinatewise log differences proves the claim.
\end{proof}
Here $r$ is the product case of the similarity of relations in B\v{e}lohl\'avek \cite[Eq.~(11), Lemma~1]{Belohlavek2007}. On a common positive support $d$ is Thompson's metric for the coordinate cone \cite{Thompson1963}. Thus the field whose entities are observer fields is symmetric and product-transitive, even if each underlying field is directed. Reapplying the same comparison on a finite collection has the same properties. They follow from the particular ratio used to compare fields.

\textit{Remark (Field updates).} Propositions~\ref{prop:pointwise-product} and~\ref{prop:convex-mixture} give two ways to change a field. Multiplication by a second field can remove entities from a fiber and never adds them. A convex step towards a fixed second field never lowers agreement with it, and repeated steps of fixed positive size converge to it entrywise. The framework permits both choices and leaves the learning rule open.

\textit{Proof.} On the common nonempty finite set of Definition~\ref{def:agreement}, put $S'=SL$ entrywise. Since $0\le L\le1$, we have $S'\le S$ and $F^{S'}_\alpha(K)\subseteq F^S_\alpha(K)$ for every $K,\alpha$. At each entry with $S>0$ the agreement factor is $L$, and at a shared zero it is one. Hence $r(S,S')\ge\min_{A,B}L(A,B)$. For $S_\lambda=(1-\lambda)S+\lambda L$, $0\le\lambda\le1$, write $r=r(S,L)$. The entrywise inequalities $rS\le L$, $rL\le S$ and $rL\le L$ give $rS_\lambda\le L$ and $rL\le S_\lambda$, hence $r(S_\lambda,L)\ge r$. Iterating with fixed $L$ and fixed $0<\lambda\le1$ gives $S_{p+1}-L=(1-\lambda)(S_p-L)$. By induction $|S_p-L|=(1-\lambda)^p|S_0-L|$ entrywise, which tends to zero.

To compare how many distinctions each observer can make on $X$, allow all entities of $X$ to receive a weight.
\begin{definition}[Richness]\label{def:richness}
Define $R(S;X)=\max_{\weights\in\Delta(X)}D_2(\weights)
=1/\min_{\weights\in\Delta(X)}\weights^{\mathsf T}\scoremat\weights$. The breadth of a generated batch is the richness of its restricted field.
\end{definition}

Section~\ref{sec:breadth} fixed the observer who scores a generated batch. We now ask which membership judgments and numerical bounds carry over when a second observer evaluates the same concept, given exemplars and outputs.
\begin{theorem}[Transfer between observers]\label{thm:transfer}
Let $S_1,S_2$ be fields on a common nonempty finite set $X$ containing $K$, with agreement $r=r(S_1,S_2)>0$. Compute both observers' quantities on $X$. Use superscript $j$ for the fiber of $S_j$, and write $T_j=T(S_j;X)$ and $T_{K,j}$ for its degree through $K$.
\begin{enumerate}
\item[(a)] $F^1_\alpha(K)\cap X\subseteq F^2_{r\alpha}(K)\cap X$ for $0\le\alpha\le1$.
\item[(b)] In particular, a generated output $E'\in X$ with $S_1(E',K)\ge\alpha$ satisfies $S_2(E',K)\ge r\alpha$.
\item[(c)] $T_2\ge r^3T_1$ and $T_{K,2}\ge r^3T_{K,1}$.
\item[(d)] Fix a nonempty batch of distinct entities in $X$ and $0<\beta\le1$. If $S_1(E_i,E_j)\ge\beta$ whenever $i\ne j$, then $D_q^{(2)}(\weights)\le1/(r\beta)$ for every weighting on the batch and every $0\le q\le\infty$.
\item[(e)] If the fields are strictly positive, each triple satisfies $r^6\rho_1\le\rho_2\le\rho_1/r^6$.
\item[(f)] $rR(S_1;X)\le R(S_2;X)\le R(S_1;X)/r$.
\end{enumerate}
Each statement also holds with the observers exchanged.
\end{theorem}
\begin{proof}
\textup{(a), (b)} Entrywise, $r S_1\le S_2\le S_1/r$. Apply the lower inequality to the score towards $K$.

\textup{(c)} On each two-step path,
\[
 S_2(A,C)\ge rS_1(A,C)\ge rT_1 S_1(A,B)S_1(B,C)
               \ge r^3T_1 S_2(A,B)S_2(B,C).
\]
Thus $r^3T_1$ is an admissible lower bound in the definition of $T_2$. Holding the middle vertex at $K$ gives the same argument with the local degrees.

\textup{(d)} The off-diagonal entries in the second field are at least $r\beta$. The diagonal is one and $0<r\beta\le1$, so Lemma~\ref{lem:diversity} applies, including to a singleton batch.

\textup{(e)} Each of the three numerator factors and each inverse denominator factor in $\rho_2/\rho_1$ belongs to $[r,1/r]$.

\textup{(f)} For the two score matrices and any probability vector on $X$, the entrywise comparison gives
\[
 r\,\weights^{\mathsf T}\scoremat_1\weights
 \le\weights^{\mathsf T}\scoremat_2\weights
 \le r^{-1}\weights^{\mathsf T}\scoremat_1\weights.
\]
The forms are positive on the simplex. Minimize there, then take reciprocals to obtain the richness bounds. Agreement is symmetric, so exchanging the observers preserves every hypothesis.
\end{proof}
The global cubic estimate in (c) is the product specialization of B\v{e}lohl\'avek's transitivity theorem \cite[Sec.~4.5, Theorem~7]{Belohlavek2007}. In a reflexive field, repeated vertices and zero two-link products contribute the neutral upper value to the infimum defining the degree. Restricting the same pointwise argument to middle vertex $K$ yields the local estimate used here.

For example, at $r=0.95$ a score of $0.90$ towards $K$ guarantees at least $0.855$ for the second observer. A score of $0.95$ guarantees at least $0.9025$, enough to pass the original $0.90$ threshold. This applies to a given exemplar and to an output of $G$. These are guarantees for the identified entities. Observer~2 may accept further inputs at $r\alpha$ on which the generator has never been assessed. A claim about the operator on all such inputs would require separate evidence. The global and local transitivity degrees each retain at least $85.7375\%$ of their original values. The diversity ceiling in (d) grows by at most $5.3\%$. The cycle-ratio multiplier lies in $[0.95^6,0.95^{-6}]$, contained in $[0.7350,1.3604]$.

\textit{Remark (An idealized model).} The second field in Theorem~\ref{thm:transfer} may be an idealized model $M$ of the observed field $S$ on a common nonempty finite set $X$. Let $K\in X$, $r=r(S,M)>0$ and $0<\alpha\le1$. If $M$ is product-transitive on $X$, both the global degree of $S$ and its degree through $K$ are at least $r^3$. For a batch with two-sided fidelity at least $\alpha$ in $S$, Theorem~\ref{thm:breadth} therefore gives $B_S(Y)\le1/(r^3\alpha^2)$. Alternatively, suppose that on $Y\cup\{K\}$ the model is a real unit-vector field satisfying Corollary~\ref{cor:gram-breadth}. One-sided fidelity $S(E_i,K)\ge\alpha$ suffices for $D_{2,S}(\weights)\le1/(r^3\alpha^2)$ for every probability vector $\weights$. At $\alpha=0.9$ the ceiling rises from about $1.235$ at $r=1$ to about $1.440$ at $r=0.95$.

\textit{Proof.} Agreement gives $rS\le M$ and $rM\le S$ entrywise. Product transitivity gives $S(A,C)\ge rM(A,B)M(B,C)\ge r^3 S(A,B)S(B,C)$, including paths through $K$. In the unit-vector case $M(E_i,K)\ge r\alpha$, so Corollary~\ref{cor:gram-breadth} gives $\weights^{\mathsf T}\scoremat_M\weights\ge r^2\alpha^2$. The observed quadratic form is at least $r$ times the model form because all $p_ip_j$ are nonnegative. Taking the reciprocal proves the claim.

The same estimates compare repeated measurements. The fraction $\kappa$ is an exception. Start with three entities and equal off-diagonal scores $1/2$. Multiply each off-diagonal score by $\exp(\varepsilon Y(A,B)/2)$ for small $\varepsilon>0$. Choose for $Y$ either a nonzero pure ranking difference, or a unit circulation around the three-cycle. The resulting fields approach the same symmetric field and their agreement tends to one, while their fractions remain zero and one, respectively.

\Needspace{9\baselineskip}
An observer's richness has simple bounds that can be read directly from its distinctions.
\begin{proposition}[Bounds on richness]\label{prop:richness}
Let $R=R(S;X)$ for a nonempty finite set $X$.
\begin{enumerate}
\item[(a)] $1\le R\le |X|$.
\item[(b)] If $0<\beta\le1$ and every off-diagonal entry is at least $\beta$, then $R\le1/\beta$.
\item[(c)] If $m\ge1$ distinct entities have every ordered off-diagonal score at most $\gamma\in[0,1]$, then $R\ge m/[1+(m-1)\gamma]$.
\item[(d)] The all-ones matrix has $R=1$, and the identity matrix has $R=|X|$.
\end{enumerate}
\end{proposition}
\begin{proof}
For (a), a point mass has diversity one and the diagonal gives $\weights^{\mathsf T}\scoremat\weights\ge\sum_i p_i^2\ge1/|X|$. Lemma~\ref{lem:diversity} gives (b). For (c), uniform mass on the $m$ entities makes the denominator at most $[1+(m-1)\gamma]/m$. For (d), the all-ones matrix gives denominator one for every weighting. For the identity matrix, the uniform weighting minimizes $\sum_i p_i^2$ at $1/|X|$.
\end{proof}
A system idealized as giving the same response to everything has richness one. A thermometer distinguishes inputs through temperature. A stone can warm, become wet or wear down, so its field depends on which response is observed. Comparisons between instruments, layers or models require the same $X$ and a comparable response protocol. Richness describes the distinctions in that response. It makes no assumption about experience or awareness.

Changes in a single observer's judgments can be tracked without assuming that learning changes the field at every step.
\begin{proposition}[Evolution as a path]\label{prop:path}
For a system-state sequence $Z_p=(X_p,S^{(p)})$, let $X$ be a nonempty finite subset of every $X_p$ and assume $S^{(p)}>0$ on $X\times X$. Put $d_p=-\log r(S^{(p)},S^{(p+1)})$ on $X$. If $\sum_p d_p<\infty$, the restricted fields converge to a positive field $S^*$. Their transitivity degrees, cycle ratios and richness converge to those of $S^*$ on $X$.
\end{proposition}
\begin{proof}
For $j>p$, every log score changes by at most $\sum_{k=p}^{j-1}d_k$. Finiteness of the total makes the log scores uniformly Cauchy on $X\times X$. Exponentiating their finite limits gives a positive field $S^*$ with unit diagonal. Taking $j\to\infty$ in the same estimate yields $-\log r(S^{(p)},S^*)\le\sum_{k=p}^\infty d_k\to0$. Theorem~\ref{thm:transfer}(c), (e) and (f), applied in both directions, gives the stated limits.
\end{proof}
Total path length gives a sufficient condition for a common portion of a field to settle. The learning process in Section~\ref{sec:learning} can instead change $G$ while leaving that portion fixed. Theorem~\ref{thm:stability} constrains a stable readout, whereas finite total path length here makes the restricted field itself converge. A percentile of the pairwise agreements can supplement the worst-case value $r$.

So far fields were compared from outside. A field may also contain a model of another field.
\begin{definition}[Model of another observer]\label{def:model}
A model of $j$ held by $i$ is a field $M_{ij}$ on $X$. Let
\[
 a_{ij}=r(S_i,M_{ij}),\qquad m_{ij}=r(M_{ij},S_j),\qquad r_{ij}=r(S_i,S_j)
\]
denote perceived agreement, modelling accuracy and actual agreement, respectively.
\end{definition}
Accurate modelling restricts the difference between perceived and actual agreement.
\begin{proposition}[Agreement through mutual modelling]\label{prop:models}
One has $r_{ij}\ge a_{ij}m_{ij}$ and $a_{ij}\ge r_{ij}m_{ij}$. When all three are positive these inequalities are equivalent to $|\log a_{ij}-\log r_{ij}|\le-\log m_{ij}$.
\end{proposition}
\begin{proof}
Apply Proposition~\ref{prop:agreement} along $S_i,M_{ij},S_j$ and along $S_i,S_j,M_{ij}$, using symmetry in the last factor. Taking logarithms yields the final assertion.
\end{proof}
An observer cannot be much mistaken about how far it agrees with another unless it models the other badly, and the bound says by how much. Accurate models in each direction constrain both observers' estimates. The proposition is conditional and does not assert that there are several observers. The construction applies to nested models or to a representation of an earlier self. \emph{Interpretation.} If several observers are present, the second-level geometry may provide a formal description of a shared world. This reading is distinct from social accounts of shared knowledge \cite{BergerLuckmann1966} and from agreement results for common-prior posteriors \cite{Aumann1976}. Convergence under mutual adjustment would require a rule for changing the models.

\subsubsection{Certification beyond tested cases}\label{sec:certification}
Agreement measured on tested entities constrains agreement on an untested one, to an extent fixed by the structure assumed.
\begin{proposition}[Extension to an untested case]\label{prop:cert-extension}
Let $S_1,S_2$ be fields on a finite set $X$ containing $a,E,K$, with $T(S_j;X)\ge T_j$ for $j=1,2$. Write $x_j=S_j(E,K)$ and $y_j=S_j(a,K)$. If $h(y_1,y_2)\ge r$ and $S_j(E,a),S_j(a,E)\ge s$ for $j=1,2$, where $r,s,T_j\in[0,1]$, then $h(x_1,x_2)\ge rT_1T_2s^2$.
\end{proposition}
\begin{proof}
Through $a$, $x_2\ge T_2s y_2\ge rT_2s y_1$. Through $E$, $y_1\ge T_1s x_1$, so $x_2\ge rT_1T_2s^2 x_1$. Exchanging the observers gives $x_1\ge rT_1T_2s^2x_2$, which also covers shared zeros.
\end{proof}
For equality on $\{a,E,K\}$ take $(y_1,x_1)=(T_1s,1)$, $(y_2,x_2)=(rT_1s,rT_1T_2s^2)$ and all other off-diagonal scores equal to $s$. The two paths into $K$ satisfy $x_j\ge T_jsy_j$ and $y_j\ge T_jsx_j$, and every other triple of distinct entities has product at most its direct score $s$. The sharp factor tends to zero as $s\downarrow0$ and is zero at $s=0$. Both local degrees $T_K$ equal one in this example, so local degrees cannot replace degrees assumed on a set that includes the untested $E$ (for product-transitive closures see \cite{DeBaetsDeMeyer2003}).

\begin{proposition}[Tests under a unit-vector representation]\label{prop:cert-gram}
Under the unit-vector representation of Corollary~\ref{cor:gram-breadth}, with vectors $u_E$ and Gram matrix $G$, suppose $G$ is known on $V=A\cup C$ and $K\notin V$. A test of $E\in V$ reveals $S(E,K)$. The entities in $A$ have been tested and the $M=|C|$ candidates in $C$ have not.
\begin{enumerate}
\item[(a)] If the entities of $H\subseteq V$ have been tested, with concept scores $z_H$, and $u_E\in\operatorname{span}\{u_h:h\in H\}$, then $S(E,K)=G_{EH}G_{HH}^{+}z_H$.
\item[(b)] One observer needs at most $d=\operatorname{rank}G_V-\operatorname{rank}G_A$ further tests to recover every score towards $K$. For two observers with rank differences $d_1,d_2$, $m=\min(M,d_1+d_2)$ shared tests suffice, each returning both scores.
\item[(c)] For all integers $0\le d_1,d_2\le M$ and every $0<\rho\le1$, deciding whether $h(S_1(E,K),S_2(E,K))\ge\rho$ for every $E\in C$ can require $m$ tests, even when all scores are positive.
\end{enumerate}
\end{proposition}
\begin{proof}
For $U=[u_h]_{h\in H}$, the orthogonal projector onto the span is $P=U(U^{\mathsf T}U)^{+}U^{\mathsf T}$, with the Moore--Penrose inverse \cite{Albert1969}. Then $u_E^{\mathsf T}Pu_K=u_E^{\mathsf T}u_K$, which is the Nystr\"om formula, exact on this span \cite{WilliamsSeeger2001}. This proves (a). For (b), the known Gram ranks allow a basis of the tested span to be extended by $d$ candidates to a basis of the span of $V$; test these, and for two observers test the union of the two extensions.

For (c), choose $C_j\subseteq C$ of sizes $d_j$ with minimum overlap, so $|C_1\cup C_2|=m$. Take orthonormal vectors $e_0$ and $e_i$, $i\in C$, and one tested entity with vector $e_0$. Field $j$ assigns $te_0+qe_i$ to $i\in C_j$ and $e_0$ to the other candidates, with $0<t<1$ and $q=\sqrt{1-t^2}$, giving rank difference $d_j$. Let the concept vector have coefficient $b$ on $e_0$ and $b(1-t)/q$ on each $e_i$, $i\in C_j$; every candidate then scores $b$ in both fields. For each $i$ in the union, an alternative changes the coefficient on $e_i$ in one field containing it to $b(\rho/2-t)/q$. A small $b>0$ leaves room to complete each concept vector to unit length along a further orthogonal axis. All scores are positive, and in the alternative only candidate $i$ fails, with ratio $\rho/2$. Until $i$ is tested, the original pair and the alternative give the same transcript for the same random bits. A strategy with zero error must therefore test every $i$ in the union almost surely \cite[cf.\ Def.~2]{HannekeYang2015}.
\end{proof}

\begin{proposition}[Rank, richness and separated exemplars]\label{prop:cert-rank}
Let $Y$ be a nonempty finite set whose score matrix $\scoremat$ has rank $k$.
\begin{enumerate}
\item[(a)] $D_2(\weights)\le R(S;Y)\le k$ for every $\weights$. This includes directed fields. For a unit-vector field $\scoremat=G_Y$.
\item[(b)] Under the hypotheses of Corollary~\ref{cor:packing}, the $n$ vectors are linearly independent, so $\operatorname{rank}G_Y=n$.
\end{enumerate}
\end{proposition}
\begin{proof}
Put $W=\operatorname{diag}(\sqrt{\weights})\,\scoremat\operatorname{diag}(\sqrt{\weights})$ and let $Q$ be the orthogonal projector onto its column space. Since $Z_{ij}^2\le Z_{ij}$,
\[
 1=\operatorname{tr}W=\langle Q,W\rangle_F\le\sqrt{\operatorname{rank}W}\,\|W\|_F\le\sqrt{k\,\weights^{\mathsf T}\scoremat\weights}.
\]
Thus $D_2(\weights)=k$ exactly when $\operatorname{rank}W=k$, $kW=Q$ and $Z_{ij}\in\{0,1\}$ whenever $p_ip_j>0$, and $R=k$ exactly when some weighting meets these conditions.

For (b), write $b_i=\langle u_i,u_K\rangle$ and $w_i=u_i-b_iu_K$. Distinct residuals have inner product at most $\gamma-\alpha^2<0$. A nonzero dependence $\sum_ic_iu_i=0$ has coefficients of both signs, since $\sum_ic_ib_i=0$ and every $b_i>0$. Its positive and negative parts are then two combinations of disjoint sets of residuals with positive weights and the same sum, whose squared norm would be strictly negative.
\end{proof}
With Corollary~\ref{cor:packing}, the largest $n$ possible when the vectors lie in a space of dimension $d\ge1$ is $\min\{d,\lfloor(1-\gamma)/(\alpha^2-\gamma)\rfloor\}$: the matrix $G_Y=(1-\gamma)I_n+\gamma\mathbf1\mathbf1^{\mathsf T}$ has rank $n$, and taking $u_K$ along the vector sum gives fidelity $\sqrt{(1+(n-1)\gamma)/n}\ge\alpha$. Independence fails at $\gamma=\alpha^2$. For $\gamma=\alpha^2=1/2$ and two orthonormal $e_1,e_2\perp u_K$, the four vectors $(u_K\pm e_i)/\sqrt2$ meet both score conditions, and the two pairs have equal sums.

\textit{Remark.} Without a margin at the threshold, approximate low rank or approximate positive semidefiniteness can still require all $M$ tests. Richness does not determine the number of tests. In Proposition~\ref{prop:cert-gram}(c) with $C_1=C_2=C$, the candidates have richness $M/[1+(M-1)t^2]$, which tends to one as $t\uparrow1$, while all $M$ tests remain necessary. Conversely, if every candidate has the same vector as some tested entity in both fields, no further test is needed, although the candidates may have richness $M$. The two structures give incomparable ranges: for symmetric links $S(a,E)=4/5$ and $S(a,K)=3/5$, the unit-vector range of $S(E,K)$ is $[0,24/25]$ and the range for transitivity degree at least $1/2$ is $[6/25,1]$. With several candidates, the interval computed for one candidate alone can overestimate its range, because every other candidate score must also be nonnegative (compare positive definite and doubly nonnegative completions \cite{GroneEtAl1984,DrewJohnson1998}).
\par\endgroup

\section{Interpretation in Machine Learning}

Similarity Field Theory provides a new geometric lens for understanding typical machine learning models. As discussed earlier in our ``Beyond a purely statistical view'' subsection, our aim here is not to rehash frequency co-occurrence arguments but to treat networks as knowledge structures of concept fibers; we instantiate this perspective with two canonical cases: neural networks and LLMs \cite{Gardenfors2000}.

\paragraph{Neural networks as composed similarity fields}
Consider a neural network trained to classify images of a cat. Let $K_{\text{Cat}}$ be the concept ``cat''. The trained network can be viewed as approximating a similarity field
\[
S_{\text{NN}}(E_{\text{input}}, K_{\text{Cat}}) \in [0,1],
\]
where $E_{\text{input}}$ is an input entity such as a vector of pixel values.

Internal activations are typically unbounded real values. For interpretation, we fix a strictly increasing function $\varphi:\mathbb R\to[0,1]$ and define a calibrated activation $a^{\sim}=\varphi(a)\in[0,1]$.
 This allows us to read a neuron’s response as a similarity value without altering the network’s forward pass. A single neuron in the first hidden layer thus induces a primitive field
\[
S_{\text{neuron}}(E_{\text{input}}, K_{\text{feature},w}) := \varphi(a(E_{\text{input}})),
\]
where $K_{\text{feature},w}$ denotes the latent feature concept associated with the neuron’s weight vector $w$. For self-similarity we explicitly normalise $S_{\text{neuron}}(K_{\text{feature},w}, K_{\text{feature},w})=1$. In practice one can pick a prototype input $E^\star$ whose calibrated activation $\varphi(a(E^\star))$ is close to the supremum and interpret it as approximately achieving similarity~$1$.

At a suitable level of abstraction, the output of a layer (or block) can be viewed as applying a map $g_\ell:[0,1]^m\to[0,1]$ to a vector of calibrated inputs, where $g_\ell$ is coordinatewise monotone and satisfies $g_\ell(1,\dots,1)=1$ on chosen prototypes. We do not require each primitive operation inside the layer to be a similarity field by itself; rather, after calibration we interpret the layer-level map as an aggregator over upstream field values. This reading applies when the aggregator is implemented by convex combinations or pointwise products of calibrated activations, which keep the range in $[0,1]$ and preserve the prototype value $1$. A layer with negative weights need not be coordinatewise monotone; the reading then holds only for the monotone part of the architecture or as an approximation. Iterating this construction over layers yields, at the level of layer outputs, a composed similarity field whose final readout approximates $S_{\text{NN}}(E_{\text{input}}, K_{\text{Cat}})$.

Training can then be viewed as similarity maximization: the network parameters are updated so that entities belonging to the cat fiber $F_\alpha(K_{\text{Cat}})$ receive higher values of $S_{\text{NN}}(E_{\text{input}}, K_{\text{Cat}})$, subject to task constraints. In this view, backpropagation shapes an internal similarity field rather than just fitting an input–output statistic.

\begin{figure}[H]
\centering
\resizebox{0.9\linewidth}{!}{
\begin{tikzpicture}[>=stealth, x=1cm, y=1cm, font=\sffamily]
\definecolor{gblue}{HTML}{1A73E8}
\tikzset{ibox/.style={draw,rounded corners=4pt,minimum width=28mm,minimum height=12mm,align=center,font=\small}, concept/.style={circle,draw=gblue,very thick,minimum size=12mm,align=center,font=\small}, val/.style={font=\small}, arr/.style={->,very thick}}
\node[ibox] (input) at (-4.2,0) {$E_{\text{input}}$};
\node[concept] (neuron) at (-0.8,0) {$K_{\text{feature},\mathbf w}$};
\draw[arr,draw=gblue] (neuron.east) -- ++(1.2,0) node[val,anchor=west] {$S(E_{\text{input}},K_{\text{feature},\mathbf w})$};
\draw[arr] (input) -- (neuron);
\end{tikzpicture}
}
\caption{A single neuron as a primitive similarity field.}
\label{fig:neuron}
\end{figure}
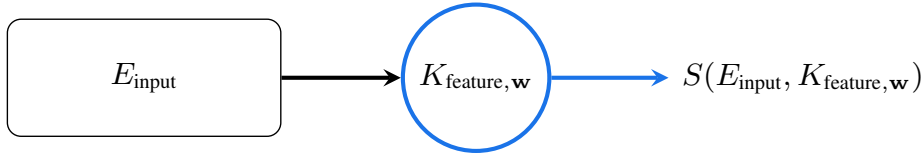
\FloatBarrier

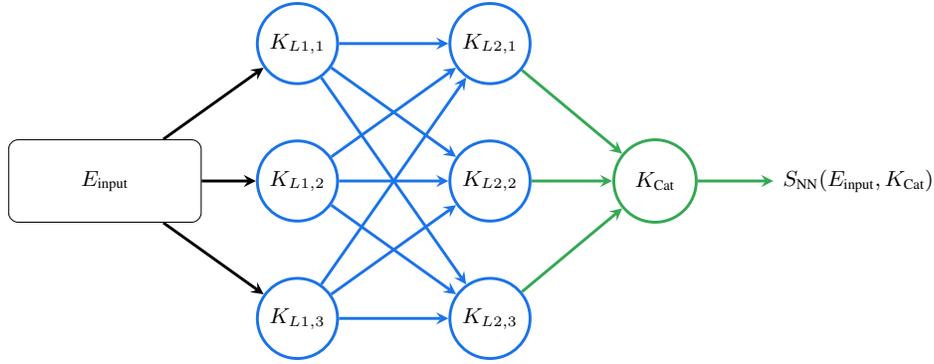
\begin{figure}[H]
\centering
\resizebox{0.9\linewidth}{!}{
\begin{tikzpicture}[>=stealth, x=1cm, y=1cm, font=\sffamily]
\definecolor{gblue}{HTML}{1A73E8}
\definecolor{ggreen}{HTML}{34A853}
\tikzset{ibox/.style={draw,rounded corners=4pt,minimum width=28mm,minimum height=12mm,align=center,font=\small}, feat/.style={circle,draw=gblue,very thick,minimum size=11mm,align=center,font=\small}, concept/.style={circle,draw=ggreen,very thick,minimum size=12mm,align=center,font=\small}, val/.style={font=\small}, arr/.style={->,very thick}}
\node[ibox] (input) at (-4.2,0) {$E_{\text{input}}$};
\node[feat] (l11) at (-1.4,2.0) {$K_{L1,1}$};
\node[feat] (l12) at (-1.4,0) {$K_{L1,2}$};
\node[feat] (l13) at (-1.4,-2.0) {$K_{L1,3}$};
\node[feat] (l21) at (1.4,2.0) {$K_{L2,1}$};
\node[feat] (l22) at (1.4,0) {$K_{L2,2}$};
\node[feat] (l23) at (1.4,-2.0) {$K_{L2,3}$};
\node[concept] (cat) at (3.8,0) {$K_{\text{Cat}}$};
\draw[arr] (input) -- (l11);
\draw[arr] (input) -- (l12);
\draw[arr] (input) -- (l13);
\draw[arr,draw=gblue] (l11) -- (l21);
\draw[arr,draw=gblue] (l11) -- (l22);
\draw[arr,draw=gblue] (l11) -- (l23);
\draw[arr,draw=gblue] (l12) -- (l21);
\draw[arr,draw=gblue] (l12) -- (l22);
\draw[arr,draw=gblue] (l12) -- (l23);
\draw[arr,draw=gblue] (l13) -- (l21);
\draw[arr,draw=gblue] (l13) -- (l22);
\draw[arr,draw=gblue] (l13) -- (l23);
\draw[arr,draw=ggreen] (l21) -- (cat);
\draw[arr,draw=ggreen] (l22) -- (cat);
\draw[arr,draw=ggreen] (l23) -- (cat);
\draw[arr,draw=ggreen] (cat.east) -- ++(1.1,0) node[val,anchor=west] {$S_{\text{NN}}(E_{\text{input}},K_{\text{Cat}})$};
\end{tikzpicture}
}
\caption{A neural network composes simple similarity fields into a complex one.}
\label{fig:network}
\end{figure}
\FloatBarrier

\paragraph{Language models as token-level similarity generators}
At the language interface, each vocabulary token can be treated as a concept $K_{\text{token}}\in U$, and an input prompt as a finite sequence of such token-concepts whose composition induces a higher-order concept $K_{\text{prompt}}$. The trained language model itself is an entity $K_{\text{Language}}$ that represents a vast meta-concept of language and, simultaneously, a generative operator $G$. For a given token concept $K_{\text{token}}$, its fiber $F_\alpha(K_{\text{token}})$ collects entities in $U$ that align with the token’s usage in human language and culture at threshold $\alpha$. When $G$ acts on a prompt, the prediction head compares the synthesized prompt concept $K_{\text{prompt}}$ against every token concept in its vocabulary; the resulting normalized probability distribution over tokens can be interpreted, after calibration, as a monotone surrogate of the similarity values $S(E,K_{\text{prompt}})$. Selecting a high-probability next token $E'_{\text{next}}$ therefore approximately enforces $S(E'_{\text{next}},K_{\text{prompt}})\ge\alpha$ for a positive effective threshold $\alpha$ set by the sampling rule, so that the model behaves as a generative operator that tends to remain within the fiber of the prompt concept under the Similarity Field Theory view.

\paragraph{A path towards interpretability}
In the SFT view, a trained network is a composition of calibrated similarity fields and can be analysed in terms of the concepts attached to its units. Each neuron computes a value $S_{\text{neuron}}(E_{\text{input}},K_{\text{latent}})$, which we interpret as the similarity between an input entity and a latent concept $K_{\text{latent}}$. In an idealised limit, the concept represented by that neuron is fully captured by its core fiber: the set of all inputs that drive its calibrated activation to the maximum value $1$,
\[
F_1(K_{\text{latent}}) = \{\,E_{\text{input}} \in U \mid S_{\text{neuron}}(E_{\text{input}},K_{\text{latent}})=1\,\}.
\]
With a strictly increasing calibration no finite activation reaches $1$, so in practice one uses the fiber at a threshold close to the largest value attained.
Recovering, even approximately, the inputs in this fiber amounts to recovering what the neuron “means”. From this perspective, interpretability becomes a geometric problem: decomposing the network into its constituent conceptual fibers and understanding how these fibers are composed across layers.

\paragraph[Cognition as Similarity Inference]{Cognition as Similarity Inference}

This interpretation suggests a deeper principle. The entities we perceive may, on this view, be understood as stable structures inferred from a dynamic field of perceived similarities. We need not assume that an entity such as a ``cat'' is given to cognition in wholly direct access; rather, cognition may operate on coherent patterns of similarity in sensory input. Even the perception of a quantity such as length can be treated as an act of relational inference. In this sense, a framework centered on similarity is not only useful for machine learning, but may also provide a suggestive model for aspects of cognition.

\subsection{Mediated Concepts and the Limits of Alignment}
A further implication of the present framework is a possible limit of alignment by description alone. If the world available to cognition is mediated through perspective, similarity relations, and stable regularities, then what models learn from human data is not the concept in itself, but an operationalized form of the concept as mediated by human observation, judgment, and description.

This is especially consequential for safety. AI systems may be aligned less to safety as such than to human understandings, descriptions, and operational criteria of safety. What humans identify, describe, or institutionalize as ``safe'' need not fully determine the underlying safety concept. Contemporary AI systems are trained on these human-interpretable and human-labelable traces of safety---such as labels, judgments, policies, and linguistic usage---rather than on safety as such. This introduces a structural risk: a system may become increasingly accurate at satisfying human-recognizable safety criteria without correspondingly becoming more closely aligned with the underlying safety concept.

\paragraph{Claim (Pseudo-Alignment).}
Improved alignment to human-recognizable safety criteria does not by itself imply improved alignment to the underlying safety concept.

For two observers, such as an evaluator and the system it evaluates, Section~\ref{sec:certification} asks how far their measured agreement extends beyond the tested cases. Under the stated product-transitivity bounds, agreement at least $r$ on a tested exemplar guarantees at least $rT_1T_2s^2$ on a candidate linked to it in both directions by scores at least $s$; no larger factor holds in general, and the guarantee vanishes as $s\to0$. Under a unit-vector representation the tested scores determine every score on the span of the tested vectors, and candidates outside it can require the number of further tests given in Proposition~\ref{prop:cert-gram}(c).

The challenge of safety alignment may therefore arise not only from limited data or computation, but also from the possibility that available descriptions of safety do not uniquely determine the target concept itself. If so, then even a model that satisfies the safety descriptions currently available to humans may remain only partially aligned with the underlying safety concept. Safety alignment is one especially consequential instance of a more general point: models are trained not on concepts in themselves, but on mediated traces of how such concepts are observed, judged, and operationalized.

\subsection*{Relations to other formalisms}

\paragraph{Similarity as relational inference}
SFT frames perception and judgement primarily as relational inference on a dynamic similarity field, without presupposing direct access to objects. In this view, stable entities appear as persistent superlevel-set structures, i.e., regularities in similarity relations that remain coherent under small perturbations. This perspective complements rather than replaces raw frequency-based descriptions by providing a geometric language for the structure that underwrites stable readouts.

\paragraph{Analogical reasoning} Analogical judgments can be expressed as structure-preserving maps between fibers. Given concepts $K_{\text{src}}$ and $K_{\text{tgt}}$, an analogy aligns subsets where $S(\cdot,K_{\text{src}})$ and $S(\cdot,K_{\text{tgt}})$ share level-set structure. This fiber-consistent view turns analogical correspondences into constraints on similarity geometry and permits quantitative tests by checking whether mapped instances remain within target fibers.

\paragraph{Symbolic constraints and fuzzy decisions}
Symbolic constraints can be cast as predicates on fibers. A discrete constraint corresponds to membership in an intersection of fibers at prescribed thresholds, while graded decisions correspond to $\alpha$-level choices on a single fiber. This recovers classical fuzzy $\alpha$-cuts within SFT, with similarity $S(\cdot,K)$ serving as the membership function and $F_{\alpha}(K)$ as its superlevel set. In this way, crisp logic, soft decisions, and their mixtures can be expressed within the same similarity-field framework, linking rule-like reasoning with graded conceptual membership \cite{Zadeh1965}.

\paragraph{Neural components under calibrated similarity}
With calibrated readouts, neural networks can be viewed as hierarchical compositions of primitive similarity fields, where intra- and inter-layer aggregations act as approximately monotone operators. Under this lens, mechanistic analyses that recover circuits, features, or probes correspond to identifying latent concepts $K$ and characterizing their fibers. When a prediction is stable, the Stability Theorem applies: the determining vectors eventually lie in every tolerance tube around the target level of the readout. This perspective places structural interpretability and readout stability within a single geometric language.

\par\begingroup
\subsection[Outlook: special cases and a complex-valued extension]{Outlook: special cases and a complex-valued extension}\label{sec:outlook}

A field first records how one entity scores against another. Additional restrictions specify which familiar mathematical models can realize those scores. The following implications concern the same set of entities.
\begin{center}
\small\begin{tabular}{@{}p{.37\linewidth}p{.56\linewidth}@{}}
\toprule
Condition on the scores & Structure obtained\\
\midrule
Product transitivity, $T=1$ & $d(A,B)=-\log S(A,B)$ is an extended directed Lawvere distance, with $-\log 0=\infty$ \cite{Lawvere1973}.\\
Symmetry and minimum transitivity & $1-S$ is a pseudo-ultrametric. At each threshold the fibers form equivalence classes, nested as the threshold varies.\\
Binary, transitive, antisymmetric and locally finite & The pairs scoring one define a locally finite partial order, the order structure used in causal sets \cite{Bombelli1987}.\\
Symmetric and positive semidefinite (PSD) & A normalized real kernel, hence a real unit-vector representation on each finite set (Proposition~\ref{prop:gram-real}) \cite{Aronszajn1950}.\\
Identity matrix and given label probabilities $\weights$ & Hill diversity, with $\log_2 D_1(\weights)=H_{\mathrm{bits}}(\weights)$ at order one \cite{Hill1973,LeinsterCobbold2012}.\\
\bottomrule
\end{tabular}
\end{center}
Here minimum transitivity means $S(A,C)\ge\min\{S(A,B),S(B,C)\}$. In the first row the triangle inequality follows by taking negative logarithms of the product inequality. For a finite field a Markov transition matrix is a related construction with row sums one. The quantities $T$, breadth, inter-observer agreement $r$ and richness $R$ are available on their stated domains before these restrictions are imposed. Theorems~\ref{thm:breadth} and~\ref{thm:transfer} retain their respective hypotheses. For the real directional analysis of Theorems~\ref{thm:cycle} and~\ref{thm:hodge} we require strictly positive scores on the comparison set, with an additional nonzero-energy condition for the fraction $\kappa$.

\begin{definition}[Complex similarity field]\label{def:complex-sft}
A complex similarity field is a map $S:U\times U\to\{z\in\mathbb C:|z|\le1\}$ satisfying $S(E,E)=1$. A modulus fiber is $F_\alpha(K)=\{E:|S(E,K)|\ge\alpha\}$. The entrywise modulus $|S|$ is a real similarity field. A real field embeds by taking all its values on the nonnegative real axis. Fibers and breadth computed with the modulus agree exactly with the corresponding quantities of $|S|$.
\end{definition}

See \hyperlink{intro-field-table}{the four-quadrant table in the Introduction}.
Each right-hand entry is contained in the entry to its left. Each real case embeds in the complex case below it. Neither a dimension nor an evolution law enters this classification. The real and complex value domains studied here serve as starting points, leaving other value domains open to investigation.

On a finite set, Hermitian positive semidefiniteness of a complex score matrix is precisely the condition for a unit-vector realization. Under these hypotheses the moduli used for membership satisfy the following bound.
\begin{proposition}[Complex unit-vector representation]\label{prop:complexgram}
On a finite set, a Hermitian PSD field has the representation $S(A,B)=\langle v_A,v_B\rangle$ by unit vectors, with the inner product conjugate-linear in its first argument. Conversely every finite family of complex unit vectors has a Hermitian PSD Gram matrix with diagonal one. Thus the characterization is exact when dimension is unrestricted. The modulus is symmetric and obeys
\[
 |S(A,C)|\ \ge\ |S(A,B)|\,|S(B,C)|
 -\sqrt{(1-|S(A,B)|^2)(1-|S(B,C)|^2)}.
\]
\end{proposition}
\begin{proof}
Factor the Hermitian PSD matrix as $V^*V$. The columns $v_E$ of $V$ have norm one by the unit diagonal. Conversely any such Gram matrix is Hermitian PSD since $x^*V^*Vx=\|Vx\|^2$. Hermitian symmetry gives $|S(A,B)|=|S(B,A)|$.
Write $v_A=\lambda v_B+a$ and $v_C=\mu v_B+c$ with $a,c\perp v_B$. In our convention $\lambda=\langle v_B,v_A\rangle$ and $\mu=\langle v_B,v_C\rangle$, so
\[
 S(A,C)=\overline{\lambda}\mu+\langle a,c\rangle,\qquad
 \|a\|^2=1-|\lambda|^2,\quad \|c\|^2=1-|\mu|^2 .
\]
The reverse triangle inequality and Cauchy--Schwarz give $|S(A,C)|\ge|\lambda||\mu|-\|a\|\|c\|$, as required.
\end{proof}

\textit{Remark.} If $S$ is the Gram matrix of complex unit vectors $k,u_1,\ldots,u_n$ and $|\langle k,u_i\rangle|\ge\alpha>0$ for every $i$, then on the modulus field $D_2(\weights)\le1/\alpha^2$ for every probability vector $\weights$ on the outputs, even if $|S|$ is not positive semidefinite. Choose the phases $\lambda_i$ as in the first line. The next two lines use the triangle inequality and Cauchy--Schwarz, respectively.
\[
\begin{aligned}
 v_i&=\lambda_i u_i,\quad |\lambda_i|=1,\quad
 \langle k,v_i\rangle=|\langle k,u_i\rangle|,\quad z=\sum_i p_i v_i,\\
 \sum_{i,j}p_i p_j|\langle u_i,u_j\rangle|
 &\ge \left|\sum_{i,j}p_i p_j\langle v_i,v_j\rangle\right|=\|z\|^2,\\
 \|z\|^2 &\ge |\langle k,z\rangle|^2
 =\left(\sum_i p_i|\langle k,u_i\rangle|\right)^2\ge\alpha^2 .
\end{aligned}
\]
The sum on the left of the second line is $1/D_2(\weights)$, which gives the bound.

\Needspace{9\baselineskip}
A three-way comparison can retain information about exemplars that the modulus-based membership test discards.
\begin{proposition}[Complex cycle ratio]\label{prop:complexcycle}
Let $S$ be a complex similarity field with no zero entries on a finite nonempty $X$. Use the oriented three-edge ratio of Definition~\ref{def:cycle}.
\begin{enumerate}
\item[(a)] $|\rho|=1$ for every triple exactly when the positive real field $|S|$ admits the detailed-balance weights of Theorem~\ref{thm:cycle}(i).
\item[(b)] If $S$ is Hermitian, write $\Delta=S(A,B)S(B,C)S(C,A)$. Then $\rho=\Delta/\overline{\Delta}$, so $|\rho|=1$ and $\arg\rho=2\arg\Delta$ modulo $2\pi$.
\end{enumerate}
\end{proposition}
\begin{proof}
Taking moduli gives the real cycle ratio of $|S|$, so Theorem~\ref{thm:cycle} proves (a). Under Hermitian symmetry the reverse product $S(A,C)S(C,B)S(B,A)$ is $\overline{\Delta}$. Dividing proves (b).
\end{proof}

For unit-vector overlaps $\Delta$ is the three-state Bargmann invariant \cite{Bargmann1964,Oszmaniec2024}. Replacing each $v_E$ by $e^{i\theta_E}v_E$ multiplies an edge by $e^{i(\theta_B-\theta_A)}$, and the three phase factors cancel around the loop. The ratio $\rho$ retains the phase of $\Delta$ modulo $\pi$. Pairwise magnitudes can therefore agree across two triples with different cycle phases.

\paragraph{Worked example.} Take
\[
 A=(1,0),\qquad B=(1,1)/\sqrt2,\qquad C_\pm=(1,\pm i)/\sqrt2 .
\]
In each triple the three pairwise transition probabilities $|\langle A,B\rangle|^2$, $|\langle B,C_\pm\rangle|^2$ and $|\langle C_\pm,A\rangle|^2$ all equal $1/2$. The complex overlaps $\langle B,C_\pm\rangle=(1\pm i)/2$ differ. Multiplying around each triangle gives $\Delta_\pm=(1\pm i)/4$, with phases $+\pi/4$ and $-\pi/4$, respectively.

The complex unit-vector case has the same mathematical form as the overlaps between unit vectors used for pure states in quantum theory. In that application an independently specified projection test for $B$ on a system prepared in $A$ succeeds with probability $|S(A,B)|^2$. These probabilities themselves form a real field. Oszmaniec et al.\ give a circuit for measuring Bargmann invariants \cite{Oszmaniec2024}. Magnot relates triangular phase defects in qubit comparisons to normalized Bargmann invariants \cite{Magnot2026}. The conditions in Proposition~\ref{prop:complexgram} concern one matrix at a time. Sequence stability is a separate condition. The present framework does not include quantum dynamics, observables, measurement or composite systems, and we make no claim about physics.

\paragraph{Open question.} Even in a vector representation the normalized nonzero sum of two representatives in a modulus fiber can leave it through cancellation of the component along $v_K$. Such a sum depends on the representatives' relative phase. The field supplies no operation for combining entities. Specifying one and relating it to intelligence with respect to $K$ are open problems.
\par\endgroup

\par\begingroup
\subsection[Predictions and possible measurements]{Predictions and possible measurements}\label{sec:pred}
This is a theory paper. The thought experiment in Section~\ref{sec:breadth} is a testable restriction on scores. We propose the following measurements for future work.

The measurements below can be read as one procedure for a system whose rules are not known. Choose a finite set of entities and record their pairwise scores under a stated protocol. Compute the transitivity degree, the cyclic fraction, the largest score between distinct labels, and the agreement with any idealized model of interest. Each quantity indicates which results apply and at what cost: Theorem~\ref{thm:breadth} through $T_K$, Theorem~\ref{thm:hodge} through $\kappa$, the entropy bound of Section~\ref{sec:breadth} through $\gamma$, and Theorem~\ref{thm:transfer} through $r$. No prior decision that the system is a metric space, a kernel or a partition is required. Several idealized models can be compared with the same field at once. Each comparison gives its own bound, all of these bounds hold together, and the tightest one applies; a model that agrees poorly gives a weak bound or none. Different models can also describe different aspects of one field, as ranking and circulation do in Theorem~\ref{thm:hodge}.

\begin{enumerate}
\item \textbf{How closely do generated batches approach the breadth ceiling?} Fix a protocol for ordered scores, a concept $K$, a batch size $n$, and both fidelity thresholds (or their common lower bound $\alpha$). Report the batch, its weights, the finite comparison set $X$ containing $K$ and the outputs, $T_K$ on $X$, the observed breadth $B$, and the finite-$n$ ceiling of Theorem~\ref{thm:breadth} when $T_K>0$. Checking the bound on the same data tests scoring and arithmetic. Both $B$ and $T_K$ use output-pair scores, while $T_K$ also uses the links to and from $K$. Actual breadth need not change monotonically with $T_K$. For any fixed batch, $\weights^{\mathsf T}\scoremat\weights=\weights^{\mathsf T}[(\scoremat+\scoremat^{\mathsf T})/2]\weights$, so its breadth depends only on the symmetric part of its score matrix. For example, take two outputs that score one in both directions with $K$. Set their mutual scores to $1/2-h$ and $1/2+h$, $0\le h\le1/2$. On these three entities $T_K=1/2-h$. For weights $(p,1-p)$ the diversity denominator is $1-p(1-p)$, minimized at $p=1/2$. Hence $B=4/3$ for every $h$. Enlarging $X$ can lower the minimum that defines $T_K$. An empirical quantile should be identified as such, since the theorem uses a lower bound on every relevant score.
\item \textbf{Do independently rated batches satisfy the vector bound?} Score the same concept and outputs with a prespecified normalized PSD kernel and with human judgments on a common, prespecified scale. For nonzero embeddings with unit directions $u,v$, $(1+\langle u,v\rangle)/2$ is the inner product of the unit vectors $(1,u)/\sqrt2$ and $(1,v)/\sqrt2$, so it is a kernel valued in $[0,1]$. Where fidelity is at least $\alpha>0$, its breadth must be at most $1/\alpha^2$. Whether the human ratings violate the same restriction is an empirical question. Record the directions of the ratings and assess their uncertainty when comparing against the threshold.
\item \textbf{How much directionality is cyclic?} Estimate directed associations on a fixed finite set with positive scores and report the gradient and cyclic energies from Theorem~\ref{thm:hodge} and, where defined, their cyclic fraction $\kappa$. Document how missing or zero entries were handled and assess sensitivity to their treatment. To see how much circulation sampling error can produce, compare the observed $\kappa$ with its distribution under a model that sets the true circulation to zero while retaining the fitted ranking and the sampling design, including unequal numbers of ratings and dependence between pairs.
\item \textbf{How well do judgments transfer between observers?} On a shared finite reference set, measure agreement $r$ between two models and test the threshold and breadth guarantees in Theorem~\ref{thm:transfer} on that set. Report the directions and pairs responsible for the limiting agreement.
\item \textbf{How many distinctions does a response make, and how does its field change?} Compare $R$ across systems or layers using specified layer responses and a common reference set. Along training, measure the discrepancy $d_p$ between consecutive restricted fields as in Proposition~\ref{prop:path}. These are proposed observations. The propositions specify how an observed discrepancy would constrain transfer.
\end{enumerate}

\paragraph{Open questions.} For product transitivity, unit-vector representations and label entropy we have stated how a conclusion weakens when its exact assumption holds only approximately. The same question can be asked of any theorem proved under an exact structural assumption: if a field meets that assumption only to a measured degree, how much of the conclusion survives? Two further questions are a rule for changing a field or an operator after new exemplars, with a proof that it improves fidelity, and how many ratings are needed to estimate $T_K$, $\kappa$ and $r$ to a given accuracy.
\par\endgroup

\section{Conclusion}

Similarity Field Theory treats dynamic similarity as the primitive for cognition and intelligence. By relaxing the strict axioms of classical mathematics and formalizing entities, fibers (as superlevel sets), and generative operators, we provide a precise language to describe the evolution of conceptual structures via sequences. The Incompatibility Theorem (Theorem~\ref{thm:asymmetric}) identifies when a pair can pass both cross-referenced thresholds. The Stability Theorem (Theorem~\ref{thm:stability}) describes how a convergent readout constrains its determining sequence. A persistent relation among changing coordinates can keep the outcome fixed, as when two oscillations cancel. The framework culminates in a generative definition of intelligence as the preservation of these structures—an operational principle that is both generalizable and measurable. Networks can be analyzed as compositions of similarity fields, enabling a principled route to interpretability. In short, this framework shifts the center of gravity from statistics to the geometry of similarity: intelligence as generative fiber preservation, and stability as eventual confinement within every tolerance of the target level. Ultimately, this work points to a shift in perspective—from objects to similarity relations—offering a computable language for modeling cognitive structure. It also suggests a possible limit of alignment by description alone: systems trained on human-mediated traces of concepts may align to human-recognizable criteria without thereby fully aligning to the underlying concepts themselves. This may be especially consequential in safety-critical settings, where improved conformity to human-recognizable safety standards need not by itself imply improved alignment to safety as such.

At positive two-sided fidelity $\alpha$ and positive degree $T_K$ of product transitivity through the concept, the breadth of a generated batch is at most $1/(T_K\alpha^2)$. On a common finite set, observers with agreement $r>0$ transfer membership of a given exemplar or a generated output at threshold $\alpha$ to membership at threshold $r\alpha$ (Theorem~\ref{thm:transfer}).

\end{document}